\pdfoutput=1
\documentclass[10pt, twocolumn, twoside]{IEEEtran}

\usepackage{multirow}
\usepackage{makecell}
\usepackage{hyperref}
\usepackage{textcomp}
\usepackage{epsfig,latexsym}
\usepackage{float}
\usepackage{indentfirst}
\usepackage{amsmath}
\usepackage{amssymb}
\usepackage{xcolor}
\usepackage{times}
\usepackage{subfigure}
\usepackage{psfrag}
\usepackage{cite}
\usepackage{flushend}
\usepackage{lastpage}
\usepackage{epstopdf}
\usepackage{fancyhdr}
\usepackage{stfloats}
\usepackage{color}
\usepackage[noend]{algpseudocode}
\usepackage{algorithmicx,algorithm}
\usepackage{graphicx}
\usepackage{tikz}
\usetikzlibrary{positioning,calc}
\definecolor{risblue}{RGB}{31,78,121}
\definecolor{risbluef}{RGB}{234,240,247}
\definecolor{risgreen}{RGB}{46,125,50}
\definecolor{risgreenf}{RGB}{233,242,233}
\definecolor{risorange}{RGB}{179,67,28}
\definecolor{risorangef}{RGB}{249,235,229}
\definecolor{risamber}{RGB}{196,150,30}
\definecolor{risamberf}{RGB}{250,245,228}
\definecolor{risgray}{RGB}{85,85,85}
\definecolor{risgrayf}{RGB}{245,245,245}

\def\diag{\textrm{diag}}

\newcommand{\bm}[1]{\mbox{\boldmath{$#1$}}}

\newcommand{\R}{\mathbb{R}}
\newcommand{\C}{\mathbb{C}}

\newtheorem{lemma}{Lemma}
\newtheorem{remark}{\bf{Remark}}
\newtheorem{proposition}{Proposition}

\begin{document}

\title{A Particle-Swarm-Assisted Gradient Meta-Learning Algorithm for Joint Transmit Precoding and STAR-RIS Coefficient Optimization}

\author{
\normalsize{\IEEEauthorblockN{Kang Zhou}
\IEEEauthorblockA{\normalsize School of Artificial Intelligence,\\
\normalsize Mianyang City College,\\
\normalsize Mianyang 621000, China\\
\normalsize zhoukang.sicnu@foxmail.com}}
}

\maketitle
\vspace{-2.0em}

\begin{abstract}
This paper investigates the joint optimization of the transmit precoder and the transmission/reflection coefficients of a simultaneously transmitting and reflecting reconfigurable intelligent surface (STAR-RIS) to maximize the weighted sum rate (WSR) in a multi-user downlink system. We propose a particle-swarm-assisted gradient meta-learning (PSA-GML) algorithm to solve this non-convex problem. Specifically, the original problem is first equivalently transformed into a tractable form by an amplitude-split parameterization and a collapsed precoder representation, which automatically satisfy the energy-conservation constraint and reduce the search dimension. Then, particle swarm optimization (PSO) performs a global search over the STAR-RIS coefficients to produce a high-quality warm start that is robust to initialization, with the transmit precoder obtained in closed form. Departing from conventional alternating optimization (AO), whose outcome is sensitive to the starting point, a coordinate-wise long short-term memory (LSTM) meta-optimizer trained by first-order gradient meta-learning is further employed to jointly refine the STAR-RIS coefficients and the transmit precoder, learning per-coordinate adaptive update rules from data. Finally, the meta-optimizer is trained offline over multiple channel realizations and applied to unseen channels without further adaptation. Numerical results show that PSA-GML attains an 11.06 bits/s/Hz WSR at a transmit SNR of 10 dB with $N=32$ elements and $K=4$ users, exceeding the conventional alternating optimization (AO) by 13.1\%---and by 6.2\% even when AO is further equipped with multiple random restarts---and the random-phase scheme by 35.1\%, with robustness against initialization. Moreover, in the interference-limited regime the learned optimizer attains a final WSR of $83.9\%$ of that reached by the hand-designed Adam refinement without any manual hyper-parameter tuning, and it transfers zero-shot across operating regimes, indicating that the learned update rule captures the intrinsic structure of the WSR landscape rather than the specifics of the training scenario.
\end{abstract}
\vspace{-0.5em}
\begin{IEEEkeywords}
STAR-RIS, weighted sum rate, joint optimization, particle swarm optimization, gradient meta-learning.
\end{IEEEkeywords}
\vspace{-0.5em}

\section{Introduction}
\vspace{0.5em}
The transition of wireless networks from 5G to 6G is propelled by use cases such as ultra-dense deployment, massive machine-type communication (mMTC), and ultra-reliable low-latency communication (URLLC), which impose growing demands for spectral efficiency, connectivity, and quality of service (QoS)~\cite{b1,b2,b3}. Massive multiple-input multiple-output (MIMO) improves spectral efficiency by deploying many antennas but incurs high hardware cost and energy consumption~\cite{b4}, motivating more cost-effective solutions. Recently, reconfigurable intelligent surfaces (RISs) have emerged as a promising technology that reconfigures the wireless propagation environment through a large number of low-cost passive reflecting elements~\cite{b5,b6,b28}. By adaptively tuning the phase shifts of the reflecting elements, an RIS can create favorable signal paths and suppress interference, thereby improving both the spectral efficiency and the energy efficiency of wireless networks at a low hardware cost~\cite{b7,b27}. However, a conventional RIS can only reflect the incident signal, which means that the transmitter and the receiver must be located on the same side of the surface. To overcome this limitation and achieve full-space $360^{\circ}$ coverage, the simultaneously transmitting and reflecting RIS (STAR-RIS) has been proposed~\cite{b8,b12}. Different from a conventional RIS, each element of a STAR-RIS can split the incident signal into a transmitted part and a reflected part, so that users on both sides of the surface can be served simultaneously~\cite{b12}. Among the various operating protocols, the energy-splitting (ES) mode is the most general one, in which the transmission and reflection amplitudes are independently controlled subject to an energy-conservation constraint~\cite{b8}.

For RIS-aided systems, the joint optimization of the transmit precoder and the passive coefficients is essential to fully exploit the achievable performance. To this end, a variety of optimization approaches have been proposed, including semidefinite relaxation (SDR)~\cite{b5}, alternating optimization (AO)~\cite{b9}, the fractional programming method~\cite{b10}, the weighted minimum mean-square-error (WMMSE) method~\cite{b11}, and the discrete phase-shift design~\cite{b29}. Although these methods achieve satisfactory performance, they generally entail high computational complexity or converge to locally optimal solutions. For STAR-RIS-aided systems, the joint transmission/reflection coefficient and beamforming design has been investigated in~\cite{b12,b13}. Joint beamforming and coefficient design for STAR-RIS assisted NOMA systems and coverage characterization of STAR-RIS networks were studied in~\cite{b14,b15}. Meanwhile, meta-heuristic algorithms such as particle swarm optimization (PSO)~\cite{b16} and genetic algorithms have been applied to RIS phase-shift design owing to their global search capability~\cite{b17}. Meta-heuristic methods, however, do not exploit the gradient information available in the problem.

More recently, learning-based methods have been developed to learn optimization strategies from data. Deep unfolding networks~\cite{b18,b19} unroll iterative optimization algorithms into layer-wise neural architectures and have achieved notable success in signal processing. As a related line, learning-to-optimize (L2O) methods train a recurrent meta-optimizer that maps gradients to parameter updates, thereby discovering a problem-specific update rule~\cite{b20,b21}. Moreover, model-based deep learning~\cite{b22} and meta-learning~\cite{b23,b24} have been exploited to improve both the adaptability and the robustness of wireless optimization algorithms. In particular, the gradient meta-learning joint optimization (GML-JO) approach in~\cite{b25} delegates the two blocks of optimization variables to a pair of learned networks and reports reduced sensitivity to the starting point. Nevertheless, GML-JO initializes the optimization from random points, which can restrict the achievable WSR on a sharply peaked landscape.

This paper studies a STAR-RIS aided multi-user downlink and formulates its WSR maximization problem, which is non-convex and intractable because the transmit precoder, the transmission/reflection amplitudes, and the phase shifts are tightly coupled under the unit-modulus and energy-conservation constraints. To tackle this problem, we develop a particle-swarm-assisted gradient meta-learning (PSA-GML) algorithm. Unlike AO methods, which are sensitive to the initialization, the proposed algorithm first employs PSO to perform a global search over the STAR-RIS coefficients, producing a high-quality warm start that is robust to initialization. Then, building on the gradient meta-learning paradigm of~\cite{b25} but departing from it, a gradient meta-learning stage employs a lightweight coordinate-wise LSTM meta-optimizer trained by first-order gradient meta-learning to jointly refine the STAR-RIS coefficients and the transmit precoder. The contributions of this paper are summarized as follows:
\begin{itemize}
\item We formulate the WSR maximization of a STAR-RIS aided multi-user downlink, where the transmit precoder and the STAR-RIS transmission/reflection coefficients are jointly optimized. Through an amplitude-split parameterization and a collapsed precoder representation, we perform an equivalent transformation of the original non-convex problem, which automatically satisfies the energy-conservation constraint and reduces the search dimension.

\item We then develop the PSA-GML algorithm. Especially, particle swarm optimization performs a global search over the STAR-RIS coefficients to produce a high-quality warm start that is robust to initialization, and a coordinate-wise LSTM meta-optimizer trained by first-order gradient meta-learning learns per-coordinate adaptive update rules to jointly refine the STAR-RIS coefficients and the precoder.

\item The designed meta-optimizer is trained offline over multiple channel realizations by unrolling the update recursion and minimizing the negative WSR averaged over the unrolled steps, which is applied to unseen channels without further adaptation, thereby avoiding the initialization sensitivity of conventional gradient-based methods.
\end{itemize}
Numerical experiments show that PSA-GML reaches an 11.06 bits/s/Hz WSR, delivering a 13.1\% performance enhancement over the benchmark AO method (6.2\% over a multiple-random-restart AO) and a 35.1\% enhancement over the random-phase scheme. Moreover, PSA-GML is largely insensitive to the starting point, and in the interference-limited regime the learned optimizer attains a final WSR of $4.73$ bits/s/Hz---$83.9\%$ of the $5.64$ bits/s/Hz reached by the hand-designed Adam refinement---while capturing about $31.6\%$ of the WSR gain that Adam realizes over the PSO-only baseline.

The remainder of this paper proceeds as follows. Section II describes the system model and the problem formulation, and Section III derives the equivalent transformations of the original problem. Section IV develops the proposed PSA-GML algorithm, Section V reports the numerical results, and Section VI draws the conclusion.

\textit{Notations:} Throughout the paper, italic letters denote scalars, bold lowercase letters denote vectors, and bold uppercase letters denote matrices; $\diag(\mathbf{a})$ constructs a diagonal matrix from the entries of $\mathbf{a}$; $(\cdot)^H$, $(\cdot)^T$, and $\|\cdot\|$ stand for the conjugate transpose, transpose, and Euclidean norm, respectively; $\mathcal{CN}(\mu,\sigma^2)$ is the complex Gaussian distribution with mean $\mu$ and variance $\sigma^2$; and $|\cdot|$ denotes the modulus of a complex scalar.

\section{System Model and Problem Formulation}
\vspace{1.0em}

As shown in Fig.~\ref{fig:sys}, we consider a STAR-RIS aided downlink communication system, where a base station (BS) equipped with $N_t$ antennas serves $K$ single-antenna users with the aid of a STAR-RIS with $N$ elements. The $K$ users are partitioned into $K_t$ transmission users and $K_r$ reflection users located on opposite sides of the STAR-RIS, with $K = K_t + K_r$. The direct links between the BS and the users are assumed to be severely blocked, so that the BS--user communication is established exclusively through the STAR-RIS. In the energy-splitting (ES) mode, the incident signal at each element is split into a transmitted part and a reflected part~\cite{b8}, and the transmission and reflection amplitudes are independently adjustable subject to an energy-conservation constraint, which is more general than the mode-switching (MS) and time-switching (TS) protocols. Let $\bm{\theta}_t = [\theta_{t,1},\dots,\theta_{t,N}]^T$ and $\bm{\theta}_r = [\theta_{r,1},\dots,\theta_{r,N}]^T$ denote the transmission and reflection phase-shift vectors, and let $\bm{\beta}_t = [\beta_{t,1},\dots,\beta_{t,N}]^T$ and $\bm{\beta}_r = [\beta_{r,1},\dots,\beta_{r,N}]^T$ denote the corresponding amplitude coefficients. The transmission and reflection coefficient vectors are then given by
\begin{equation}
\mathbf{v}_t = \bm{\beta}_t \odot e^{j\bm{\theta}_t}, \quad
\mathbf{v}_r = \bm{\beta}_r \odot e^{j\bm{\theta}_r},
\label{eq:coeff}
\end{equation}
where $\odot$ denotes the element-wise product. The amplitude coefficients satisfy the energy-conservation constraint
\begin{equation}
\beta_{t,n}^2 + \beta_{r,n}^2 = 1, \quad \forall n = 1,\dots,N.
\label{eq:energy}
\end{equation}

Let $\mathbf{G} \in \C^{N\times N_t}$ denote the BS-to-RIS channel, and let $\mathbf{H}_t \in \C^{K_t\times N}$ and $\mathbf{H}_r \in \C^{K_r\times N}$ denote the RIS-to-transmission-user and RIS-to-reflection-user channels, respectively. Both the i.i.d. Rayleigh and the Saleh-Valenzuela (SV) channel models are considered~\cite{b26}. For the SV channel, the BS-to-RIS channel is modeled as
\begin{equation}
\mathbf{G} = \sqrt{\frac{N N_t}{N_c N_{\mathrm{ray}}}} \sum_{c=1}^{N_c} \sum_{\ell=1}^{N_{\mathrm{ray}}} \alpha_{c,\ell}\, \mathbf{a}_R(\varphi_{c,\ell}) \mathbf{a}_T^{H}(\psi_{c,\ell}),
\label{eq:svG}
\end{equation}
where $N_c$ and $N_{\mathrm{ray}}$ denote the numbers of clusters and rays per cluster, respectively, $\alpha_{c,\ell}\sim\mathcal{CN}(0,1)$ is the complex path gain, $\varphi_{c,\ell}$ and $\psi_{c,\ell}$ are the angle of arrival (AoA) at the STAR-RIS and the angle of departure (AoD) at the BS, respectively, and $\mathbf{a}_R \in \C^{N}$ and $\mathbf{a}_T \in \C^{N_t}$ are the steering vectors of the $N$-element and $N_t$-element half-wavelength-spaced uniform linear arrays (ULAs), respectively, with
\begin{equation}
\mathbf{a}_M(\varphi) = \frac{1}{\sqrt{M}} \big[1, e^{j\pi\sin\varphi}, \dots, e^{j\pi(M-1)\sin\varphi}\big]^T.
\label{eq:steering}
\end{equation}
The RIS-to-user channel admits the analogous SV form, i.e., the channel between the STAR-RIS and the $k$-th user is given by
\begin{equation}
\mathbf{h}_k = \sqrt{\frac{N}{N_c N_{\mathrm{ray}}}} \sum_{c=1}^{N_c} \sum_{\ell=1}^{N_{\mathrm{ray}}} \alpha_{c,\ell,k}\, \mathbf{a}_N(\varphi_{c,\ell,k}),
\label{eq:svh}
\end{equation}
where $\mathbf{h}_k^{H}$ is the $k$-th row of $\mathbf{H}_t$ or $\mathbf{H}_r$. For the i.i.d. Rayleigh channel, all entries of $\mathbf{G}$, $\mathbf{H}_t$, and $\mathbf{H}_r$ are drawn independently from $\mathcal{CN}(0,1)$. All channels are normalized to unit variance, with the large-scale path loss absorbed into the transmit SNR. The effective channel from the BS to the $k$-th user is the $k$-th row of the composite channel matrix
\begin{equation}
\mathbf{A} = \begin{bmatrix}
\mathbf{H}_t \diag(\mathbf{v}_t) \mathbf{G} \\
\mathbf{H}_r \diag(\mathbf{v}_r) \mathbf{G}
\end{bmatrix} \in \C^{K \times N_t},
\label{eq:A}
\end{equation}
where the first $K_t$ rows correspond to the transmission users and the remaining $K_r$ rows correspond to the reflection users. Let $\mathbf{w}_k \in \C^{N_t}$ denote the transmit precoding vector for the $k$-th user. The received signal at the $k$-th user is
\begin{equation}
y_k = \mathbf{a}_k^H \mathbf{w}_k s_k + \sum_{j\neq k} \mathbf{a}_k^H \mathbf{w}_j s_j + n_k,
\label{eq:rx}
\end{equation}
where $\mathbf{a}_k^H$ is the $k$-th row of $\mathbf{A}$, $s_k$ is the data symbol for the $k$-th user with $\mathbb{E}[|s_k|^2]=1$, and $n_k \sim \mathcal{CN}(0,\sigma^2)$ is the additive white Gaussian noise. Accordingly, the signal-to-interference-plus-noise ratio (SINR) of the $k$-th user is
\begin{equation}
\gamma_k = \frac{|\mathbf{a}_k^H \mathbf{w}_k|^2}{\sum_{j\neq k} |\mathbf{a}_k^H \mathbf{w}_j|^2 + \sigma^2}.
\label{eq:sinr}
\end{equation}
The WSR is then defined as
\begin{equation}
R(\mathbf{W}, \bm{\Phi}) = \sum_{k=1}^{K} \omega_k \log_2 (1 + \gamma_k),
\label{eq:wsr}
\end{equation}
where $\omega_k > 0$ is the priority weight of the $k$-th user, $\mathbf{W} = [\mathbf{w}_1,\dots,\mathbf{w}_K]$ collects the precoding vectors, and $\bm{\Phi}$ collects all STAR-RIS coefficients. Our objective is to maximize the WSR subject to the transmit-power and energy-conservation constraints, which is formulated as
\begin{subequations}\label{eq:p1}
\begin{align}
& \max_{\mathbf{W}, \bm{\beta}_t, \bm{\beta}_r, \bm{\theta}_t, \bm{\theta}_r} \; R(\mathbf{W}, \bm{\Phi}) \label{eq:p1_obj} \\
& \text{s.t.} \;\; \sum_{k=1}^{K}\|\mathbf{w}_k\|^2 \le P_{\max}, \label{eq:p1_c1} \\
& \qquad \beta_{t,n}^2 + \beta_{r,n}^2 = 1, \ \beta_{t,n}, \beta_{r,n} \ge 0, \ \forall n, \label{eq:p1_c2} \\
& \qquad \theta_{t,n}, \theta_{r,n} \in [0, 2\pi), \ \forall n. \label{eq:p1_c3}
\end{align}
\end{subequations}

Problem~\eqref{eq:p1} is non-convex due to the fractional form of the SINR~\eqref{eq:sinr}, the unit energy constraint~\eqref{eq:p1_c2}, and the strong coupling between the precoder and the STAR-RIS coefficients, which renders conventional convex optimization inapplicable. To gain further insight, consider the Lagrangian associated with the power constraint~\eqref{eq:p1_c1}:
\begin{equation}
\mathcal{L}(\mathbf{W}, \bm{\Phi}, \mu) = -R(\mathbf{W}, \bm{\Phi}) + \mu\Big(\textstyle\sum_{k} \|\mathbf{w}_k\|^2 - P_{\max}\Big),
\label{eq:lagrangian}
\end{equation}
where $\mu \ge 0$ is the dual variable. The stationarity condition $\partial \mathcal{L}/\partial \mathbf{w}_k = \mathbf{0}$ couples all precoders through the interference terms of~\eqref{eq:sinr}, so no closed-form precoder exists in general, and the dual gap is generally nonzero owing to the non-convexity. We therefore avoid Lagrangian dualization and instead enforce the constraints exactly by reparameterization, as detailed in Section III.

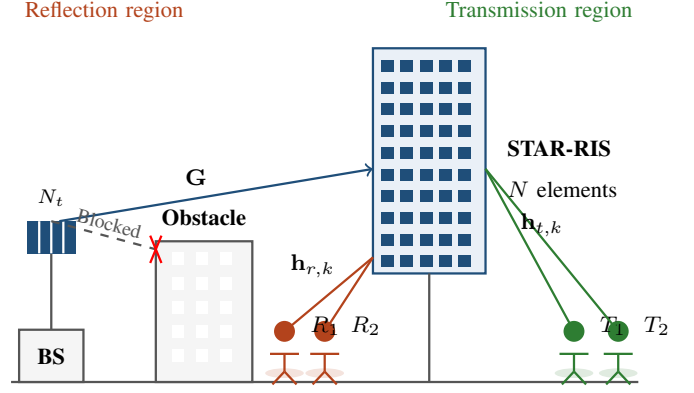
\begin{figure}[t]
\centering
\resizebox{\columnwidth}{!}{%
\begin{tikzpicture}[font=\footnotesize]
\def\person#1#2#3{%
  \fill[#3!15] (#1,#2-0.52) ellipse (0.24 and 0.06);
  \fill[#3] (#1,#2) circle (0.13);
  \draw[#3, thick] (#1-0.18,#2-0.28) -- (#1+0.18,#2-0.28);
  \draw[#3, thick] (#1,#2-0.28) -- (#1,#2-0.50);
  \draw[#3, thick] (#1,#2-0.50) -- (#1-0.15,#2-0.66);
  \draw[#3, thick] (#1,#2-0.50) -- (#1+0.15,#2-0.66);
}

\draw[risgray, thick] (0.4,0.35) -- (8.2,0.35);

\fill[risgrayf, draw=risgray, thick] (0.5,0.35) rectangle (1.3,1.0);
\node at (0.9,0.67) {\textbf{BS}};
\draw[risgray, thick] (0.9,1.0) -- (0.9,2.3);
\fill[risblue] (0.6,1.95) rectangle (1.2,2.35);
\draw[white, thick] (0.75,1.95) -- (0.75,2.35);
\draw[white, thick] (0.9,1.95) -- (0.9,2.35);
\draw[white, thick] (1.05,1.95) -- (1.05,2.35);
\node[anchor=south, font=\scriptsize] at (0.9,2.42) {$N_t$};

\fill[risgrayf] (2.2,0.35) rectangle (3.4,2.1);
\draw[risgray, thick] (2.2,0.35) rectangle (3.4,2.1);
\foreach \wy in {0.6,0.9,1.2,1.5,1.8}{
  \foreach \wx in {2.4,2.7,3.0}{
    \fill[white] (\wx,\wy) rectangle (\wx+0.14,\wy+0.16);
  }
}
\node at (2.8,2.4) {\textbf{Obstacle}};

\fill[risbluef] (4.9,1.7) rectangle (6.3,4.5);
\draw[risblue, thick] (4.9,1.7) rectangle (6.3,4.5);
\foreach \i in {0,...,4}{
  \foreach \j in {0,...,9}{
    \fill[risblue] (5.0+0.24*\i, 1.78+0.27*\j) rectangle (5.16+0.24*\i, 1.92+0.27*\j);
  }
}
\draw[risgray, thick] (5.6,1.7) -- (5.6,0.35);
\node[anchor=west] at (6.45,3.25) {\textbf{STAR-RIS}};
\node[anchor=west] at (6.45,2.75) {$N$ elements};

\person{3.8}{0.98}{risorange}
\person{4.3}{0.98}{risorange}
\node[anchor=west] at (4.0,1.05) {$R_1$};
\node[anchor=west] at (4.5,1.05) {$R_2$};

\person{7.4}{0.98}{risgreen}
\person{7.95}{0.98}{risgreen}
\node[anchor=west] at (7.6,1.05) {$T_1$};
\node[anchor=west] at (8.15,1.05) {$T_2$};

\draw[->, thick, draw=risblue] (1.0,2.35) -- (4.9,3.0);
\node at (2.7,2.9) {$\mathbf{G}$};

\draw[->, thick, draw=risgreen] (6.3,3.0) -- (7.4,0.98);
\draw[->, thick, draw=risgreen] (6.3,3.0) -- (7.95,0.98);
\node at (7.0,2.3) {$\mathbf{h}_{t,k}$};

\draw[->, thick, draw=risorange] (4.9,1.9) -- (3.8,0.98);
\draw[->, thick, draw=risorange] (4.9,1.9) -- (4.3,0.98);
\node at (4.15,1.8) {$\mathbf{h}_{r,k}$};

\draw[thick, dashed, draw=risgray] (0.9,2.35) -- (2.2,2.0);
\draw[thick, draw=red] (2.13,1.83) -- (2.27,2.17);
\draw[thick, draw=red] (2.27,1.83) -- (2.13,2.17);
\node[rotate=-15, font=\scriptsize, text=risgray] at (1.62,2.32) {Blocked};

\node[text=risorange, anchor=west] at (0.45,4.95) {Reflection region};
\node[text=risgreen, anchor=east] at (8.25,4.95) {Transmission region};

\end{tikzpicture}%
}
\caption{Downlink STAR-RIS aided multi-user communication system, where the BS serves the transmission users $T_1,T_2$ and the reflection users $R_1,R_2$ exclusively through the STAR-RIS, since the direct BS--user links are blocked by the obstacle.}
\label{fig:sys}
\end{figure}

\section{Equivalent Transformations of the Original Problem}
\vspace{1.0em}

In this section, we transform the original non-convex problem~\eqref{eq:p1} into a sequence of equivalent tractable forms, which lays the foundation for the proposed PSA-GML algorithm.

\subsection{Amplitude-Split Parameterization}
The energy-conservation constraint~\eqref{eq:p1_c2} couples the transmission and reflection amplitudes of each element. The following lemma shows that this constraint can be absorbed by a smooth amplitude-split parameterization.
\begin{lemma}\label{lem:amp}
For any feasible amplitude pair $(\beta_{t,n}, \beta_{r,n})$ satisfying $\beta_{t,n}^2 + \beta_{r,n}^2 = 1$ with $\beta_{t,n}, \beta_{r,n} \ge 0$, there exists a unique angle $\varphi_n \in [0, \pi/2]$ such that $\beta_{t,n} = \cos\varphi_n$ and $\beta_{r,n} = \sin\varphi_n$, and vice versa. Hence, the constraint~\eqref{eq:p1_c2} is equivalent to the amplitude-split parameterization
\begin{equation}
\beta_{t,n} = \cos\varphi_n, \quad \beta_{r,n} = \sin\varphi_n, \quad \varphi_n \in [0, \pi/2],
\label{eq:phi}
\end{equation}
which absorbs the energy-conservation constraint automatically.
\end{lemma}
\begin{IEEEproof}
Since $\beta_{t,n}^2 + \beta_{r,n}^2 = 1$ and $\beta_{t,n}, \beta_{r,n} \ge 0$, the point $(\beta_{t,n}, \beta_{r,n})$ lies on the unit circle in the first quadrant. Therefore, there exists a unique angle $\varphi_n = \arctan(\beta_{r,n}/\beta_{t,n}) \in [0, \pi/2]$ (with $\varphi_n = \pi/2$ when $\beta_{t,n} = 0$) satisfying $\beta_{t,n} = \cos\varphi_n$ and $\beta_{r,n} = \sin\varphi_n$. Conversely, for any $\varphi_n \in [0, \pi/2]$, the identity $\cos^2\varphi_n + \sin^2\varphi_n = 1$ holds, so the resulting pair is feasible. \hfill $\blacksquare$
\end{IEEEproof}

With Lemma~\ref{lem:amp}, the optimization variables reduce to the amplitude-split angles $\bm{\varphi} = [\varphi_1,\dots,\varphi_N]^T$, the phase shifts $\bm{\theta}_t,\bm{\theta}_r$, and the precoder $\mathbf{W}$, and problem~\eqref{eq:p1} is equivalently reformulated as
\begin{subequations}\label{eq:p2}
\begin{align}
& \max_{\mathbf{W}, \bm{\varphi}, \bm{\theta}_t, \bm{\theta}_r} \; R(\mathbf{W}, \bm{\Phi}) \label{eq:p2_obj} \\
& \text{s.t.} \;\; \sum_{k=1}^{K}\|\mathbf{w}_k\|^2 \le P_{\max}, \label{eq:p2_c1} \\
& \qquad \varphi_n \in [0, \pi/2], \ \theta_{t,n}, \theta_{r,n} \in [0, 2\pi), \ \forall n. \label{eq:p2_c2}
\end{align}
\end{subequations}

\subsection{Closed-Form Zero-Forcing Precoder}
For a fixed set of STAR-RIS coefficients, the zero-forcing (ZF) precoder admits a closed form and nulls all inter-user interference, as stated below.
\begin{proposition}\label{prop:zf}
For a fixed composite channel $\mathbf{A} \in \C^{K \times N_t}$ with full row rank (which holds when $K \le N_t$), the zero-forcing (ZF) precoder
\begin{equation}
\mathbf{W}_{\mathrm{ZF}} = \mathbf{A}^H (\mathbf{A}\mathbf{A}^H)^{-1},
\label{eq:zf}
\end{equation}
satisfies $\mathbf{a}_k^H \mathbf{w}_j = 0$ for all $j \neq k$, and reduces the SINR~\eqref{eq:sinr} to the interference-free form $\gamma_k = |\mathbf{a}_k^H \mathbf{w}_k|^2/\sigma^2$.
\end{proposition}
\begin{IEEEproof}
Substituting~\eqref{eq:zf} into the composite channel yields $\mathbf{A}\mathbf{W}_{\mathrm{ZF}} = \mathbf{A}\mathbf{A}^H(\mathbf{A}\mathbf{A}^H)^{-1} = \mathbf{I}_K$, where $\mathbf{I}_K$ is the $K\times K$ identity matrix. Consequently, the $(k,j)$-th entry satisfies $\mathbf{a}_k^H\mathbf{w}_j = [\mathbf{A}\mathbf{W}_{\mathrm{ZF}}]_{k,j} = \delta_{k,j}$, which vanishes for $j\neq k$ and equals one for $j = k$. Substituting into~\eqref{eq:sinr} gives $\gamma_k = |\mathbf{a}_k^H\mathbf{w}_k|^2/\sigma^2$. \hfill $\blacksquare$
\end{IEEEproof}
Note that in practical use the ZF precoder is power-normalized to satisfy the power constraint~\eqref{eq:p1_c1}, i.e., $\mathbf{W}_{\mathrm{ZF}}\leftarrow\sqrt{P_{\max}}\,\mathbf{W}_{\mathrm{ZF}}/\|\mathbf{W}_{\mathrm{ZF}}\|_F$, which scales the useful signal $\mathbf{a}_k^H\mathbf{w}_k$ uniformly and preserves the zero-forcing property.

\subsection{Collapsed Precoder Representation}
To jointly optimize the STAR-RIS coefficients and the precoder with a reduced search dimension, we adopt the collapsed precoder representation.
\begin{lemma}\label{lem:collapse}
For a fixed composite channel $\mathbf{A} \in \C^{K\times N_t}$ with full row rank, it is without loss of optimality to restrict the transmit precoder to the collapsed form
\begin{equation}
\mathbf{W} = \mathbf{A}^H \mathbf{X},
\label{eq:collapse}
\end{equation}
where $\mathbf{X} \in \C^{K\times K}$ is a collapsed variable, since the null-space component of $\mathbf{W}$ does not affect the SINR and only increases the transmit power.
\end{lemma}
\begin{IEEEproof}
Any precoder admits the orthogonal decomposition $\mathbf{W} = \mathbf{A}^H\mathbf{X} + \mathbf{W}_{\perp}$ with $\mathbf{W}_{\perp} \in \mathrm{null}(\mathbf{A})$, i.e., $\mathbf{A}\mathbf{W}_{\perp} = \mathbf{0}$. Then $\mathbf{A}\mathbf{W} = \mathbf{A}\mathbf{A}^H\mathbf{X}$, so the received signal~\eqref{eq:rx} and the SINR~\eqref{eq:sinr} depend on $\mathbf{X}$ alone and are independent of $\mathbf{W}_{\perp}$. On the other hand, since $\mathbf{A}^H\mathbf{X} \in \mathrm{col}(\mathbf{A}^H)$ is orthogonal to $\mathrm{null}(\mathbf{A})$, the transmit power satisfies $\|\mathbf{W}\|_F^2 = \|\mathbf{A}^H\mathbf{X}\|_F^2 + \|\mathbf{W}_{\perp}\|_F^2$, which is minimized by $\mathbf{W}_{\perp} = \mathbf{0}$. Therefore, for any feasible $\mathbf{W}$, the collapsed form~\eqref{eq:collapse} achieves the same WSR with no larger transmit power. \hfill $\blacksquare$
\end{IEEEproof}
\begin{remark}
Lemma~\ref{lem:collapse} reduces the dimension of the precoder optimization from $N_t K$ to $K^2$ complex entries, which is significant when the number of BS antennas $N_t$ is large. The collapsed variable is initialized as the ZF solution $\mathbf{X}^{(0)} = (\mathbf{A}\mathbf{A}^H)^{-1}$.
\end{remark}

With Lemma~\ref{lem:collapse}, problem~\eqref{eq:p2} is further reformulated as
\begin{subequations}\label{eq:p3}
\begin{align}
& \max_{\mathbf{X}, \bm{\varphi}, \bm{\theta}_t, \bm{\theta}_r} \; R(\mathbf{A}^H\mathbf{X}, \bm{\Phi}) \label{eq:p3_obj} \\
& \text{s.t.} \;\; \|\mathbf{A}^H\mathbf{X}\|_F^2 \le P_{\max}, \label{eq:p3_c1} \\
& \qquad \varphi_n \in [0, \pi/2], \ \theta_{t,n}, \theta_{r,n} \in [0, 2\pi), \ \forall n. \label{eq:p3_c2}
\end{align}
\end{subequations}

\subsection{Unconstrained Reparameterization}
To enable end-to-end gradient propagation, the bounded amplitude-split angles are reparameterized through the sigmoid function.
\begin{proposition}\label{prop:sigmoid}
The box constraint $\varphi_n \in [0, \pi/2]$ is enforced by the unconstrained reparameterization
\begin{equation}
\varphi_n = \frac{\pi}{2} \varsigma(\tilde{\varphi}_n) = \frac{\pi/2}{1 + e^{-\tilde{\varphi}_n}},
\label{eq:sigmoid}
\end{equation}
where $\tilde{\varphi}_n \in \R$ is an unconstrained variable and $\varsigma(\cdot)$ denotes the sigmoid function. With~\eqref{eq:sigmoid}, the composite channel~\eqref{eq:A} becomes fully differentiable with respect to $\tilde{\bm{\varphi}}$, $\bm{\theta}_t$, and $\bm{\theta}_r$.
\end{proposition}
\begin{IEEEproof}
The sigmoid function $\varsigma(x) = 1/(1+e^{-x})$ is a strictly increasing bijection from $\R$ onto the open interval $(0,1)$, so $\frac{\pi}{2}\varsigma(\tilde{\varphi}_n)$ is a bijection from $\R$ onto $(0,\pi/2)$. Hence, every feasible $\varphi_n \in [0,\pi/2]$ is reachable (the endpoints $0$ and $\pi/2$ are approached in the limit $|\tilde{\varphi}_n|\to\infty$), and every $\tilde{\varphi}_n$ produces a feasible angle. Since the map is smooth, the composite channel~\eqref{eq:A} is differentiable in $\tilde{\bm{\varphi}}$. \hfill $\blacksquare$
\end{IEEEproof}
Note that the phase shifts $\bm{\theta}_t$ and $\bm{\theta}_r$ require no reparameterization, since they enter the composite channel~\eqref{eq:A} only through $e^{j\theta_{t,n}}$ and $e^{j\theta_{r,n}}$, which are $2\pi$-periodic; any real-valued phase is therefore equivalent to one in $[0,2\pi)$, so the phases can be optimized directly as unconstrained variables.

\section{Proposed PSA-GML Algorithm}
\vspace{1.0em}

In this section, we present the proposed PSA-GML algorithm, which combines global exploration and local refinement. The key observation is that the non-convex WSR landscape of~\eqref{eq:p3} contains many local optima, so a purely gradient-based method is sensitive to initialization, whereas a purely heuristic method does not exploit gradient information. The proposed algorithm therefore consists of two stages: a PSO-based global warm start (Stage 1), followed by a gradient meta-learning refinement (Stage 2).

\begin{figure*}[t]
\centering
\begin{tikzpicture}[
  conn/.style={draw, ->, thick},
  inode/.style={circle, draw, thick, fill=white, inner sep=1.5pt, font=\small},
  gate/.style={rectangle, rounded corners=2pt, draw=risgray, fill=risgrayf, thick,
               minimum width=1.35cm, minimum height=0.72cm, align=center, inner sep=1.5pt, font=\footnotesize},
  op/.style={circle, draw=risgray, fill=white, thick, inner sep=1.5pt, font=\small},
]

\node[inode, draw=risamber, text=risamber] (cprev) at (0,5.3) {$c_d^{(s-1)}$};
\node[op] (xf) at (2.4,5.3) {$\odot$};
\node[op] (add) at (5.0,5.3) {$+$};
\node[inode, draw=risamber, text=risamber] (cnow) at (7.0,5.3) {$c_d^{s}$};
\node[gate] (tanhc) at (9.0,5.3) {$\tanh$};
\node[op] (xo) at (11.0,5.3) {$\odot$};
\node[inode, draw=risgreen, text=risgreen] (hnow) at (13.0,5.3) {$h_d^{s}$};

\node[gate] (gf) at (2.4,3.0) {$f=\varsigma$};
\node[gate] (gi) at (4.0,3.0) {$i=\varsigma$};
\node[gate] (gc) at (6.0,3.0) {$\tilde{c}=\tanh$};
\node[gate] (go) at (11.0,3.0) {$o=\varsigma$};

\node[op] (xi) at (5.0,4.2) {$\odot$};

\node[inode, draw=risblue, text=risblue] (g) at (0,0.6) {$g_d$};
\node[inode, draw=risgreen, text=risgreen] (hprev) at (0,1.7) {$h_d^{(s-1)}$};
\node[risgray, font=\footnotesize] at (0,0.06) {gradient};
\node[risgray, font=\footnotesize, anchor=east] at (-0.72,1.7) {hidden};

\coordinate (M) at (1.3,1.15);
\draw[->, risblue] (g) -- (M);
\draw[->, risgreen] (hprev) -- (M);
\draw[thick] (M) -- (11.0,1.15);
\draw[->, thick] (2.4,1.15) -- (2.4,2.62);
\draw[->, thick] (4.0,1.15) -- (4.0,2.62);
\draw[->, thick] (6.0,1.15) -- (6.0,2.62);
\draw[->, thick] (11.0,1.15) -- (11.0,2.62);
\node[risgray, font=\footnotesize] at (7.6,1.5) {$[\,g_d;\ h_d^{(s-1)}\,]$};

\draw[conn, risamber] (cprev) -- (xf);
\draw[conn] (xf) -- (add);
\draw[conn, risamber] (add) -- (cnow);
\draw[conn] (cnow) -- (tanhc);
\draw[conn] (tanhc) -- (xo);
\draw[conn, risgreen] (xo) -- (hnow);

\draw[conn, risgray] (gf) -- (xf);
\draw[conn, risgray] (gi) -- (xi);
\draw[conn, risgray] (gc) -- (xi);
\draw[conn, risgray] (xi) -- (add);
\draw[conn, risgray] (go) -- (xo);

\node[inode, draw=risorange, text=risorange] (dx) at (15.4,5.3) {$\Delta\xi_d$};
\draw[conn, risorange] (hnow) -- (dx);
\node[risgray, font=\footnotesize] at (14.2,5.62) {readout};

\draw[->, dashed, risamber] (cnow.north) -- (7.0,6.6) -- (0,6.6) -- (cprev.north);
\node[risamber, font=\footnotesize] at (3.5,6.8) {cell state $c_d$ carried};
\draw[->, dashed, risgreen] (hnow.south) -- (13.0,-0.6) -- (-0.35,-0.6) -- (-0.35,1.7) -- (hprev.west);
\node[risgreen, font=\footnotesize] at (6.3,-0.85) {hidden state $h_d$ carried over $S$ steps};

\end{tikzpicture}
\caption{Neural-network structure of the coordinate-wise LSTM learned optimizer $\mathcal{M}_\psi$. For each coordinate $d$, the preprocessed gradient $g_d$ and the recurrent hidden state $h_d^{(s-1)}$ are concatenated and mapped through four gates---forget $f$, input $i$, candidate $\tilde{c}$, and output $o$---via weights $\mathbf{W}_*\in\mathbb{R}^{H\times 1}$ (on $g_d$) and $\mathbf{U}_*\in\mathbb{R}^{H\times H}$ (on $h_d$). The gates update the cell state $c_d$ and the hidden state $h_d$, which are carried across the $S$ refinement steps (dashed loops); a readout layer then produces the adaptive per-coordinate update $\Delta\xi_d=\nu\,(\mathbf{a}^T h_d^s+b)$.}
\label{fig:arch}
\end{figure*}
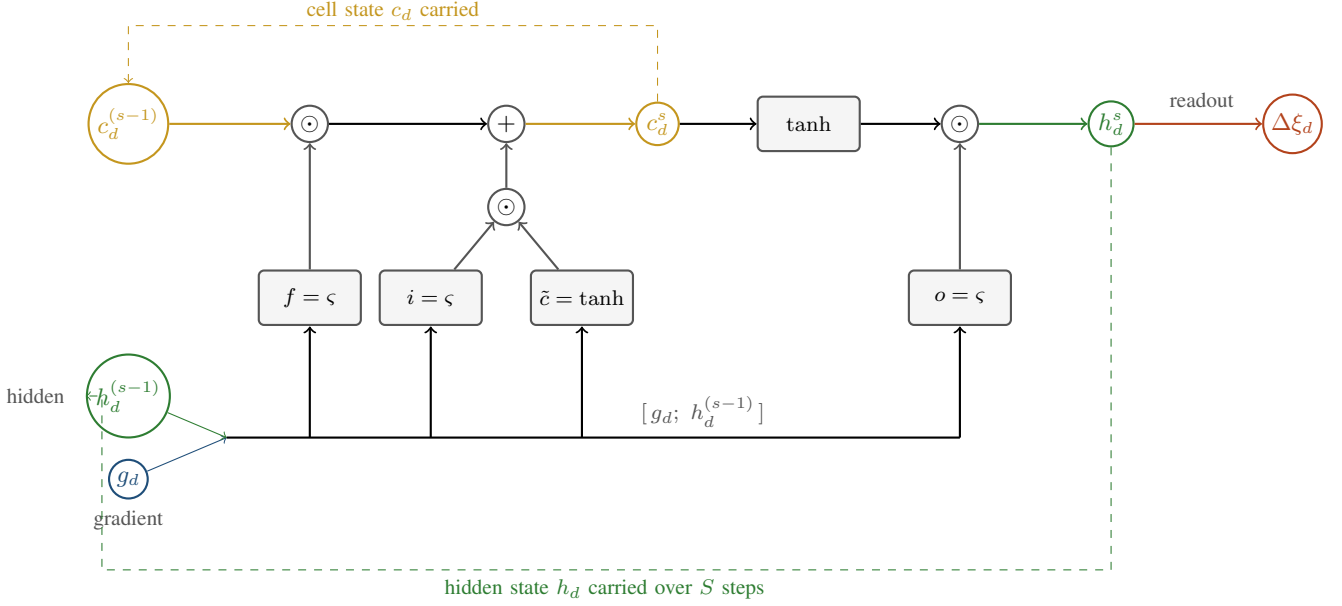

\subsection{PSO-Based Global Warm Start}
In the first stage, PSO performs a global search over the STAR-RIS coefficients. Each particle encodes a candidate solution
\begin{equation}
\mathbf{x} = [\bm{\varphi}; \bm{\theta}_t; \bm{\theta}_r] \in \mathcal{S},
\label{eq:particle}
\end{equation}
where the search space is the box $\mathcal{S} = [0,\pi/2]^{N} \times [0,2\pi)^{N} \times [0,2\pi)^{N}$, so that the amplitude-split angles and the phase shifts respect~\eqref{eq:phi} and~\eqref{eq:p1_c3} by construction. The fitness of a particle is the WSR achieved by the closed-form ZF precoder~\eqref{eq:zf} evaluated at the composite channel $\mathbf{A}(\mathbf{x})$:
\begin{equation}
f(\mathbf{x}) = R\big(\mathbf{W}_{\mathrm{ZF}}(\mathbf{x}), \mathbf{x}\big) = \sum_{k=1}^{K} \omega_k \log_2\!\big(1 + \gamma_k(\mathbf{x})\big),
\label{eq:psofit}
\end{equation}
where $\gamma_k(\mathbf{x})$ is obtained by substituting $\mathbf{A}(\mathbf{x})$ and the normalized ZF precoder into~\eqref{eq:sinr}. This keeps the fitness consistent with the final objective~\eqref{eq:p3} while avoiding an inner precoder optimization.

Let $\mathbf{x}_i^{(t)}$ and $\mathbf{v}_i^{(t)}$ denote the position and velocity of the $i$-th particle at iteration $t$, and let $\mathbf{p}_i$ and $\mathbf{g}$ denote the personal best and global best positions, respectively. The velocity and position are updated according to
\begin{align}
\mathbf{v}_i^{(t+1)} &= \omega \mathbf{v}_i^{(t)} + c_1 r_1 (\mathbf{p}_i - \mathbf{x}_i^{(t)}) + c_2 r_2 (\mathbf{g} - \mathbf{x}_i^{(t)}), \label{eq:vup}\\
\mathbf{x}_i^{(t+1)} &= \Pi_{\mathcal{S}}\!\big(\mathbf{x}_i^{(t)} + \mathbf{v}_i^{(t+1)}\big),
\label{eq:xup}
\end{align}
where $\omega$ is the inertia weight, $c_1$ and $c_2$ are the acceleration coefficients, $r_1, r_2 \sim \mathcal{U}(0,1)$ are random numbers, and $\Pi_{\mathcal{S}}(\cdot)$ denotes the element-wise projection (clipping) onto the box $\mathcal{S}$. The personal and global best positions are updated as
\begin{align}
\mathbf{p}_i &\leftarrow \mathbf{x}_i^{(t+1)} \ \text{if}\ \ f\big(\mathbf{x}_i^{(t+1)}\big) > f(\mathbf{p}_i), \label{eq:pbest}\\
\mathbf{g} &\leftarrow \arg\max_{i} f(\mathbf{p}_i), \label{eq:gbest}
\end{align}
so that the global-best fitness is never degraded, i.e., $f(\mathbf{g}^{(t+1)}) \ge f(\mathbf{g}^{(t)})$ for all $t$. After $I_{\mathrm{PSO}}$ iterations, the global best $\mathbf{g}$ is taken as the warm-start point for the refinement stage. Since PSO explores the entire feasible region $\mathcal{S}$, the resulting warm start is largely insensitive to the random initialization, which addresses the initialization sensitivity of gradient-based methods.

\subsection{Gradient Meta-Learning Refinement}
Starting from the PSO warm start, the second stage refines the solution through a gradient meta-learning backbone, i.e., a learned optimizer that maps the gradient of the objective~\eqref{eq:wsr} to an adaptive update. At the warm start, $\mathbf{X}^{(0)} = (\mathbf{A}\mathbf{A}^H)^{-1}$ reproduces exactly the ZF precoder of Stage~1, so the two stages share the same objective at initialization; as $\mathbf{X}$ evolves, the collapsed precoder~\eqref{eq:collapse} becomes a strict generalization of the ZF solution. With the collapsed precoder~\eqref{eq:collapse} and the unconstrained reparameterization~\eqref{eq:sigmoid}, the optimization reduces to updating the vector
\begin{equation}
\bm{\xi} = [\tilde{\bm{\varphi}}; \bm{\theta}_t; \bm{\theta}_r; \mathrm{vec}(\mathbf{X})] \in \R^{D},
\label{eq:xi}
\end{equation}
jointly, with $D = 3N + 2K^2$, where the complex precoder $\mathbf{X}$ is represented by its real and imaginary parts. At each refinement step, the variables are reconstructed, the precoder is normalized to satisfy the power constraint~\eqref{eq:p1_c1} by the differentiable scaling $\mathbf{W} \leftarrow \mathbf{W}\sqrt{P_{\max}/\sum_k\|\mathbf{w}_k\|^2}$, the WSR~\eqref{eq:wsr} is computed, and the gradient $\nabla_{\bm{\xi}} R$ is obtained by automatic differentiation through this scaling.

Since each entry of the composite channel $A_{k,m} = \sum_{n} H_{k,n} v_n G_{n,m}$ is a smooth function of the coefficients, $R$ is differentiable in $\bm{\xi}$. Differentiating~\eqref{eq:wsr} and~\eqref{eq:sinr} yields the scalar derivative
\begin{equation}
\frac{\partial R}{\partial \gamma_k} = \frac{\omega_k}{\ln 2\,(1+\gamma_k)},
\label{eq:grad_gamma}
\end{equation}
and the chain rule
\begin{equation}
\nabla_{\bm{\xi}} R = \sum_{k=1}^{K} \frac{\partial R}{\partial \gamma_k} \nabla_{\bm{\xi}} \gamma_k
\label{eq:grad_chain}
\end{equation}
closes the gradient, where each $\gamma_k$ is a smooth ratio of quadratics in $\mathbf{A}$ and $\mathbf{W}$. The dependence of the composite channel on the coefficients is rank-one in each element, i.e.,
\begin{equation}
\frac{\partial A_{k,m}}{\partial \varphi_n} = H_{k,n} \frac{\partial v_n}{\partial \varphi_n} G_{n,m}, \quad
\frac{\partial A_{k,m}}{\partial \theta_n} = j\, H_{k,n} v_n G_{n,m},
\label{eq:grad_A}
\end{equation}
with $\partial v_{t,n}/\partial\varphi_n = -\sin\varphi_n\, e^{j\theta_{t,n}}$ and $\partial v_{r,n}/\partial\varphi_n = \cos\varphi_n\, e^{j\theta_{r,n}}$ for the transmission and reflection coefficients, respectively, and $\partial A_{k,m}/\partial\theta_n = j\,H_{k,n}v_{t,n}G_{n,m}$ for $k\le K_t$ or $j\,H_{k,n}v_{r,n}G_{n,m}$ for $k>K_t$, which renders the backward pass fully explicit.

Instead of a hand-designed optimizer, the variables are updated by a \emph{learned optimizer}:
\begin{equation}
\bm{\xi}^{(s+1)} = \bm{\xi}^{(s)} + \mathcal{M}_\psi\!\left(\nabla_{\bm{\xi}} R(\bm{\xi}^{(s)})\right),
\label{eq:l2o}
\end{equation}
where $\mathcal{M}_\psi(\cdot)$ is a coordinate-wise long short-term memory (LSTM) meta-optimizer parameterized by $\psi$~\cite{b20}, whose neural-network structure is illustrated in Fig.~\ref{fig:arch}. In a coordinate-wise LSTM, the same cell is applied independently to every scalar coordinate of the gradient, and the hidden state carried by each coordinate stores its own momentum-like memory, thereby playing the role of the per-coordinate adaptive statistics that make Adam robust. To account for the distinct scales of the STAR-RIS coefficients and the precoder, two separate LSTMs are employed, i.e., $\mathcal{M}_c$ for $[\tilde{\bm{\varphi}}; \bm{\theta}_t; \bm{\theta}_r]$ and $\mathcal{M}_x$ for $\mathrm{vec}(\mathbf{X})$, and the gradient is preprocessed by the symmetric logarithmic mapping
\begin{equation}
g_d = \mathrm{sgn}\!\left(\nabla_{\xi_d} R\right)\,\log\!\left(1 + \left|\nabla_{\xi_d} R\right|\right),
\label{eq:preproc}
\end{equation}
to handle its wide dynamic range. Concretely, for the $d$-th optimization coordinate with preprocessed gradient $g_d$, the meta-optimizer maintains hidden and cell state vectors $\mathbf{h}_d, \mathbf{c}_d \in \R^{H}$ (with hidden size $H$) and updates them according to
\begin{align}
\mathbf{f}_d &= \varsigma(\mathbf{W}_f g_d + \mathbf{U}_f \mathbf{h}_d + \mathbf{b}_f), \label{eq:lstm}\\
\mathbf{i}_d &= \varsigma(\mathbf{W}_i g_d + \mathbf{U}_i \mathbf{h}_d + \mathbf{b}_i), \nonumber\\
\tilde{\mathbf{c}}_d &= \tanh(\mathbf{W}_c g_d + \mathbf{U}_c \mathbf{h}_d + \mathbf{b}_c), \nonumber\\
\mathbf{o}_d &= \varsigma(\mathbf{W}_o g_d + \mathbf{U}_o \mathbf{h}_d + \mathbf{b}_o), \nonumber\\
\mathbf{c}_d &\leftarrow \mathbf{f}_d \odot \mathbf{c}_d + \mathbf{i}_d \odot \tilde{\mathbf{c}}_d, \nonumber\\
\mathbf{h}_d &\leftarrow \mathbf{o}_d \odot \tanh(\mathbf{c}_d), \nonumber
\end{align}
and outputs the coordinate update $\Delta\xi_d = \nu\,(\mathbf{a}^T \mathbf{h}_d + b)$ with output scale $\nu$, where $\mathbf{a}\in\mathbb{R}^{H}$ and $b$ are the readout weight and bias. The parameters $\{\mathbf{W}_{\ast}, \mathbf{U}_{\ast}, \mathbf{b}_{\ast}, \mathbf{a}, b\}$ are shared across all coordinates, so the per-coordinate memory resides solely in the recurrent state $(\mathbf{h}_d, \mathbf{c}_d)$, which mimics the momentum and running statistics of a hand-designed optimizer such as Adam~\cite{b20}.

The meta-optimizer is trained \emph{offline} over a set of channel realizations by unrolling the recursion~\eqref{eq:l2o} for $T$ steps and minimizing the negative average WSR:
\begin{equation}
\min_{\psi} \; \frac{1}{|\mathcal{C}|} \sum_{c \in \mathcal{C}} \left[ -\frac{1}{T}\sum_{s=1}^{T} R_c\!\left(\bm{\xi}^{(s)}\right) \right],
\label{eq:meta}
\end{equation}
where $\mathcal{C}$ is the training set of channel realizations and $\bm{\xi}^{(0)}$ is the PSO warm start of channel $c$. The average loss in~\eqref{eq:meta} encourages both fast convergence and a high final WSR. Notably, since the WSR landscape at the warm start is sharply peaked---the ZF precoder drives the inter-user interference nearly to zero, so any perturbation sharply degrades the SINR---a second-order backpropagation through~\eqref{eq:l2o} is numerically unstable. We therefore detach the gradient $\nabla_{\bm{\xi}} R$ from the computational graph during meta-training, which corresponds to a first-order (truncated-backpropagation) approximation in the spirit of first-order MAML (FOMAML)~\cite{b23} and renders the training stable while the LSTM still learns per-coordinate adaptive update rules. Fig.~\ref{fig:train} summarizes the complete offline meta-training procedure.
\begin{remark}
Different from the model-agnostic meta-learning (MAML) algorithm, which differentiates through the entire inner-loop optimization, the proposed PSA-GML detaches the gradient input during meta-training. This first-order approximation avoids the second-order Hessian-vector products that are numerically unstable on the sharply peaked WSR landscape, while retaining the per-coordinate adaptivity that is essential for convergence.
\end{remark}

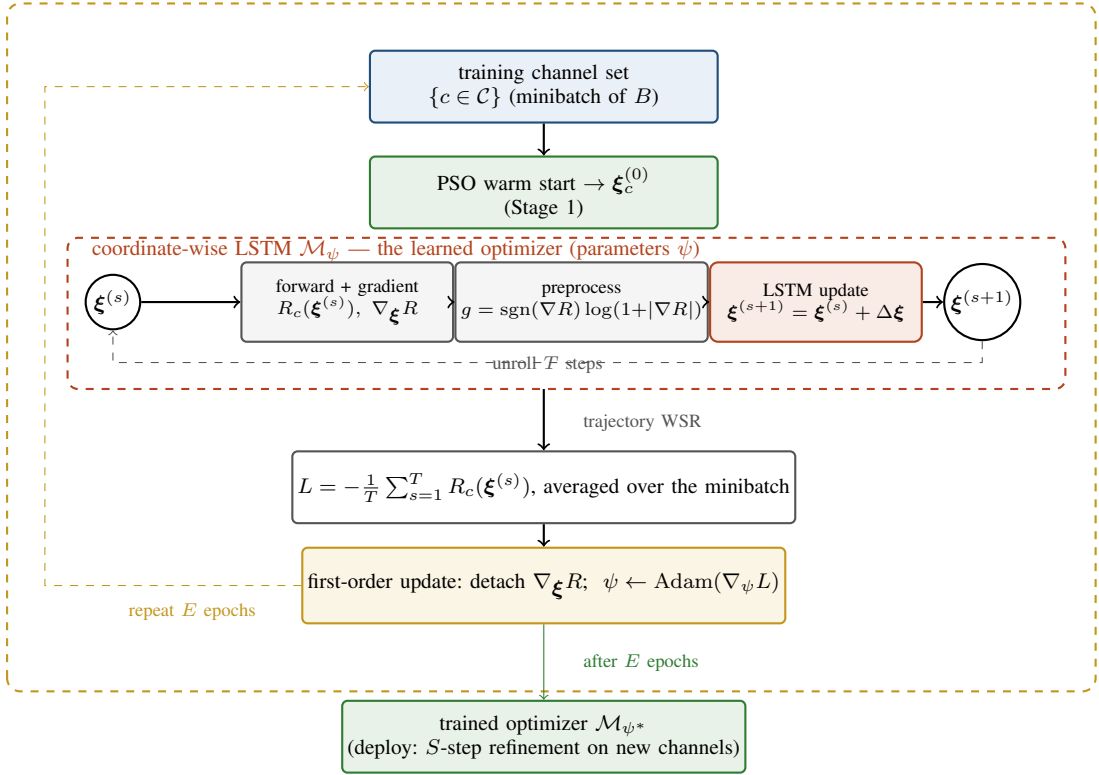
\begin{figure*}[t]
\centering
\begin{tikzpicture}[
  tdata/.style={rectangle, rounded corners=2pt, draw=risblue, fill=risbluef, thick,
                minimum width=4.6cm, minimum height=0.95cm, align=center, inner sep=2pt, font=\footnotesize},
  tps/.style={rectangle, rounded corners=2pt, draw=risgreen, fill=risgreenf, thick,
              minimum width=4.6cm, minimum height=0.95cm, align=center, inner sep=2pt, font=\footnotesize},
  tcomp/.style={rectangle, rounded corners=2pt, draw=risgray, fill=risgrayf, thick,
                minimum width=2.8cm, minimum height=1.05cm, align=center, inner sep=2pt, font=\scriptsize},
  tlstm/.style={rectangle, rounded corners=3pt, draw=risorange, fill=risorangef, thick,
                minimum width=2.8cm, minimum height=1.05cm, align=center, inner sep=2pt, font=\scriptsize},
  tloss/.style={rectangle, rounded corners=2pt, draw=risgray, fill=white, thick,
                minimum width=5.6cm, minimum height=0.95cm, align=center, inner sep=2pt, font=\footnotesize},
  tupd/.style={rectangle, rounded corners=2pt, draw=risamber, fill=risamberf, thick,
               minimum width=6.0cm, minimum height=1.0cm, align=center, inner sep=2pt, font=\footnotesize},
  tcirc/.style={circle, draw, thick, fill=white, inner sep=1.3pt, font=\footnotesize},
  tarr/.style={->, thick},
]

\node[risamber, font=\small] at (7.5,10.0) {Offline meta-training: gradient meta-learning ($E$ epochs)};

\draw[risamber, dashed, thick, rounded corners=5pt] (0.4,0.3) rectangle (14.8,9.4);

\node[tdata] (ch) at (7.5,8.3) {training channel set\\ $\{c\in\mathcal{C}\}$ (minibatch of $B$)};
\node[tps] (pso) at (7.5,6.9) {PSO warm start $\rightarrow \bm{\xi}_c^{(0)}$\\ (Stage 1)};
\draw[tarr] (ch.south) -- (pso.north);

\draw[risorange, dashed, thick, rounded corners=4pt] (1.2,4.3) rectangle (14.4,6.3);
\node[risorange, font=\footnotesize, anchor=west] at (1.4,6.12) {coordinate-wise LSTM $\mathcal{M}_{\psi}$ --- the learned optimizer (parameters $\psi$)};

\node[tcirc] (x0) at (1.8,5.45) {$\bm{\xi}^{(s)}$};
\node[tcomp] (fwd) at (4.9,5.45) {forward + gradient\\ $R_c(\bm{\xi}^{(s)}),\ \nabla_{\bm{\xi}}R$};
\node[tcomp] (pre) at (8.0,5.45) {preprocess\\ $g=\mathrm{sgn}(\nabla R)\log(1{+}|\nabla R|)$};
\node[tlstm] (lstm) at (11.1,5.45) {LSTM update\\ $\bm{\xi}^{(s+1)}=\bm{\xi}^{(s)}+\Delta\bm{\xi}$};
\node[tcirc] (x1) at (13.3,5.45) {$\bm{\xi}^{(s+1)}$};

\draw[tarr] (x0) -- (fwd);
\draw[tarr] (fwd) -- (pre);
\draw[tarr] (pre) -- (lstm);
\draw[tarr] (lstm) -- (x1);

\draw[->, dashed, risgray] (x1.south) -- (13.3,4.65) -- (1.8,4.65) -- (x0.south);
\node[risgray, font=\scriptsize] at (7.55,4.62) {unroll $T$ steps};

\node[tloss] (loss) at (7.5,3.0) {$L=-\frac{1}{T}\sum_{s=1}^{T}R_c(\bm{\xi}^{(s)})$, averaged over the minibatch};
\draw[tarr] (7.5,4.3) -- (loss.north);
\node[risgray, font=\scriptsize, anchor=west] at (7.9,3.85) {trajectory WSR};

\node[tupd] (upd) at (7.5,1.7) {first-order update: detach $\nabla_{\bm{\xi}}R$;\ \ $\psi\leftarrow\mathrm{Adam}(\nabla_{\psi}L)$};
\draw[tarr] (loss.south) -- (upd.north);

\draw[->, dashed, risamber] (upd.west) -- (0.9,1.7) -- (0.9,8.3) -- (ch.west);
\node[risamber, font=\scriptsize] at (2.85,1.35) {repeat $E$ epochs};

\node[tps] (out) at (7.5,-0.3) {trained optimizer $\mathcal{M}_{\psi^{*}}$\\ (deploy: $S$-step refinement on new channels)};
\draw[->, risgreen] (upd.south) -- (out.north);
\node[risgreen, font=\scriptsize, anchor=west] at (7.9,0.68) {after $E$ epochs};

\end{tikzpicture}
\caption{Offline meta-training of the coordinate-wise LSTM learned optimizer $\mathcal{M}_\psi$. \emph{Inner loop:} for each training channel, the PSO warm start $\bm{\xi}_c^{(0)}$ is unrolled for $T$ steps through the LSTM, producing a WSR trajectory. \emph{Outer loop:} the meta-loss $-\frac{1}{T}\sum_s R_c(\bm{\xi}^{(s)})$ is averaged over a minibatch and minimized by a first-order (FOMAML) update of the LSTM parameters $\psi$, repeated for $E$ epochs; the trained optimizer $\mathcal{M}_{\psi^{*}}$ is then deployed for $S$-step refinement on new channels.}
\label{fig:train}
\end{figure*}

\subsection{Proposed PSA-GML Algorithm}
The overall procedure of the proposed PSA-GML algorithm is summarized in Algorithm~1.

\begin{algorithm}[tp]
\caption{Proposed PSA-GML Algorithm}
\label{alg:psagml}
\begin{algorithmic}[1]
\Procedure{PSA-GML}{$\mathbf{G}, \mathbf{H}_t, \mathbf{H}_r, \{\omega_k\}, \sigma^2, P_{\max}$}
    \State Initialize $P$ particles $\{\mathbf{x}_i^{(0)}\}$ and velocities $\{\mathbf{v}_i^{(0)}\}$ randomly in $\mathcal{S}$
    \State Initialize personal best $\mathbf{p}_i \gets \mathbf{x}_i^{(0)}$ and global best $\mathbf{g}$
    \For{$t \gets 1$ \textbf{to} $I_{\mathrm{PSO}}$}
        \For{each particle $i$}
            \State Construct $\mathbf{A}(\mathbf{x}_i)$ via \eqref{eq:A} and $\mathbf{W}_{\mathrm{ZF}}$ via \eqref{eq:zf}
            \State Compute fitness $f(\mathbf{x}_i)$ via \eqref{eq:psofit}
            \State Update $\mathbf{p}_i$ and $\mathbf{g}$ via \eqref{eq:pbest}--\eqref{eq:gbest}
        \EndFor
        \For{each particle $i$}
            \State Update $\mathbf{v}_i$ and $\mathbf{x}_i$ via \eqref{eq:vup}--\eqref{eq:xup}
        \EndFor
    \EndFor
    \State Initialize $\bm{\varphi}^{(0)} \gets \mathbf{g}$; map to unconstrained $\tilde{\bm{\varphi}}$ via \eqref{eq:sigmoid}; set $\mathbf{X}^{(0)} \gets (\mathbf{A}\mathbf{A}^H)^{-1}$
    \State Initialize $R_{\mathrm{best}} \gets -\infty$
    \For{$s \gets 1$ \textbf{to} $S$}
        \State Reconstruct $\bm{\varphi} = \frac{\pi}{2}\varsigma(\tilde{\bm{\varphi}})$; form $\mathbf{A}$ and $\mathbf{W} = \mathbf{A}^H\mathbf{X}$; normalize $\mathbf{W}$ to satisfy \eqref{eq:p1_c1}
        \State Compute $R$ via \eqref{eq:wsr} and $\nabla_{\bm{\xi}} R$ via \eqref{eq:grad_gamma}--\eqref{eq:grad_A}
        \State \textbf{if} $R > R_{\mathrm{best}}$ \textbf{then} $R_{\mathrm{best}} \gets R$, $\bm{\xi}_{\mathrm{best}} \gets \bm{\xi}$
        \State Update $\bm{\xi} \gets \bm{\xi} + \mathcal{M}_{\psi}\!\left(\nabla_{\bm{\xi}} R\right)$ via \eqref{eq:l2o}
    \EndFor
    \State \Return $\mathbf{W}^{\ast}$, $(\bm{\varphi}^{\ast}, \bm{\theta}_t^{\ast}, \bm{\theta}_r^{\ast})$ reconstructed from $\bm{\xi}_{\mathrm{best}}$
\EndProcedure
\end{algorithmic}
\end{algorithm}

\subsection{Convergence and Complexity Analysis}
We first establish a monotone non-decreasing property of the proposed algorithm, which is a weaker guarantee than convergence to a stationary point but ensures that the returned solution is no worse than the warm start.
\begin{proposition}\label{prop:mono}
The proposed PSA-GML algorithm produces a WSR no smaller than the WSR of its PSO warm start.
\end{proposition}
\begin{IEEEproof}
In Stage~1, the global best is updated only when a strictly larger fitness is found, i.e., $f(\mathbf{g}^{(t+1)}) = \max\big\{f(\mathbf{g}^{(t)}), \max_i f(\mathbf{x}_i^{(t+1)})\big\} \ge f(\mathbf{g}^{(t)})$. In Stage~2, the refinement is initialized at the warm start---the coefficients are taken from $\mathbf{g}$ and the collapsed precoder is the ZF solution $\mathbf{X}^{(0)}=(\mathbf{A}\mathbf{A}^H)^{-1}$, so that $R(\bm{\xi}^{(0)}) = f(\mathbf{g}^{(I_{\mathrm{PSO}})})$---and the best iterate is retained along the trajectory, giving $\max_s R(\bm{\xi}^{(s)}) \ge R(\bm{\xi}^{(0)})$. Combining the two stages completes the proof. \hfill $\blacksquare$
\end{IEEEproof}

We then analyze the per-evaluation computational complexity of the proposed algorithm and the AO baseline. For the PSO stage, each particle evaluation involves constructing the composite channel $\mathbf{A}$ in~\eqref{eq:A}, computing the ZF precoder in~\eqref{eq:zf}, and evaluating the WSR in~\eqref{eq:wsr}, which costs $\mathcal{O}(NN_t K + K^3)$ dominated by the matrix inversion. With $P$ particles and $I_{\mathrm{PSO}}$ iterations, the total complexity of Stage~1 is $\mathcal{O}(P I_{\mathrm{PSO}} (NN_t K + K^3))$. For Stage~2, each unrolled step involves the forward WSR evaluation $\mathcal{O}(NN_t K + K^3)$ together with the coordinate-wise LSTM update $\mathcal{O}(D H^2)$ with hidden size $H$, giving $\mathcal{O}(S(NN_t K + K^3 + D H^2))$ in total. We note that the LSTM term is in fact the dominant per-step cost---with $N=32$, $N_t=32$, $K=4$, and $H=32$, one has $D=3N+2K^2=128$ and $D H^2\approx1.3\times10^5$, whereas $NN_t K + K^3\approx4.1\times10^3$. In contrast, each iteration of the AO baseline alternately updates the STAR-RIS coefficients by a gradient-ascent step and the precoder by the regularized ZF solution~\cite{b9}, both of which share the same order $\mathcal{O}(NN_t K + K^3)$ dominated by the $K\times K$ matrix inversion. Hence, the per-iteration complexity of the proposed refinement and that of the AO baseline are of the same order, and the practical computational advantage of PSA-GML stems from the parallelism of the $P$ particles in Stage~1 and from its initialization robustness, which removes the need for the multiple random restarts that AO would otherwise require, rather than from a lower asymptotic per-iteration cost. Summarizing, the overall complexity of the proposed algorithm is
\begin{align}
\mathcal{C}_{\mathrm{PSA\text{-}GML}} = \mathcal{O}\!\big(&P I_{\mathrm{PSO}}(NN_t K + K^3) \nonumber\\
&+ S(NN_t K + K^3 + D H^2)\big),
\label{eq:complexity}
\end{align}
where $D = 3N + 2K^2$ and $H$ is the LSTM hidden size.

\section{Simulation Results}
\vspace{1.0em}

In this section, we evaluate the performance of the proposed PSA-GML algorithm in STAR-RIS aided multi-user downlink systems through simulation experiments.

\subsection{Parameter Setting and Baseline}
Unless otherwise stated, the BS is equipped with $N_t = 32$ antennas and the STAR-RIS has $N = 32$ elements serving $K_t = 2$ transmission users and $K_r = 2$ reflection users. The noise power is set to $\sigma^2 = 1$, and the transmit SNR is $10$ dB. Both i.i.d. Rayleigh channels and Saleh-Valenzuela (SV) channels with $N_c = 3$ clusters and $N_{\mathrm{ray}} = 8$ rays are considered. The PSO parameters are set as $P = 20$, $I_{\mathrm{PSO}} = 50$, $\omega = 0.7$, and $c_1 = c_2 = 1.5$, and the refinement stage uses $S = 300$ steps of the learned optimizer with hidden size $H = 32$ and output scale $\nu = 0.05$, trained offline over $120$ channel realizations via first-order gradient meta-learning before deployment. All results are averaged over independent channel realizations, with $10$ realizations for the SNR and user sweeps, $8$ for the element sweep, $16$ for the interference-limited study, and $20$ for the convergence behavior.

We compare the proposed PSA-GML algorithm with the following baselines:
\begin{itemize}
\item \textbf{Random phase + ZF:} The STAR-RIS coefficients are randomly generated, and the precoder is the closed-form ZF solution~\eqref{eq:zf}.
\item \textbf{AO/BCD:} An alternating optimization (block coordinate descent) baseline that alternately updates the STAR-RIS coefficients by gradient ascent and the precoder by the regularized ZF solution until convergence~\cite{b9}. It is initialized from random phases and iterated until the WSR increment falls below a small tolerance or a maximum number of iterations $N_{\mathrm{AO}}$ is reached. Because the AO outcome is sensitive to its starting point, in addition to the single-run AO depicted in the figures, we also report the best result over $20$ random restarts, whose total online budget is no smaller than that of the proposed algorithm, to ensure a fair comparison.
\item \textbf{PSO only:} The PSO stage of the proposed algorithm without the refinement stage.
\item \textbf{Refine only:} The gradient refinement stage initialized from random coefficients rather than the PSO warm start.
\end{itemize}

\subsection{Convergence Behavior}
Fig.~\ref{fig:conv} shows the convergence behavior of the proposed PSA-GML algorithm and the baseline methods. It is observed that the proposed algorithm converges rapidly to a high WSR within a small number of refinement steps and starts from a higher initial WSR than the refine-only baseline. In contrast, the refine-only baseline, which is initialized from random coefficients, converges to a substantially lower WSR.

\begin{figure}[t]
\centering
\includegraphics[width=0.5\textwidth]{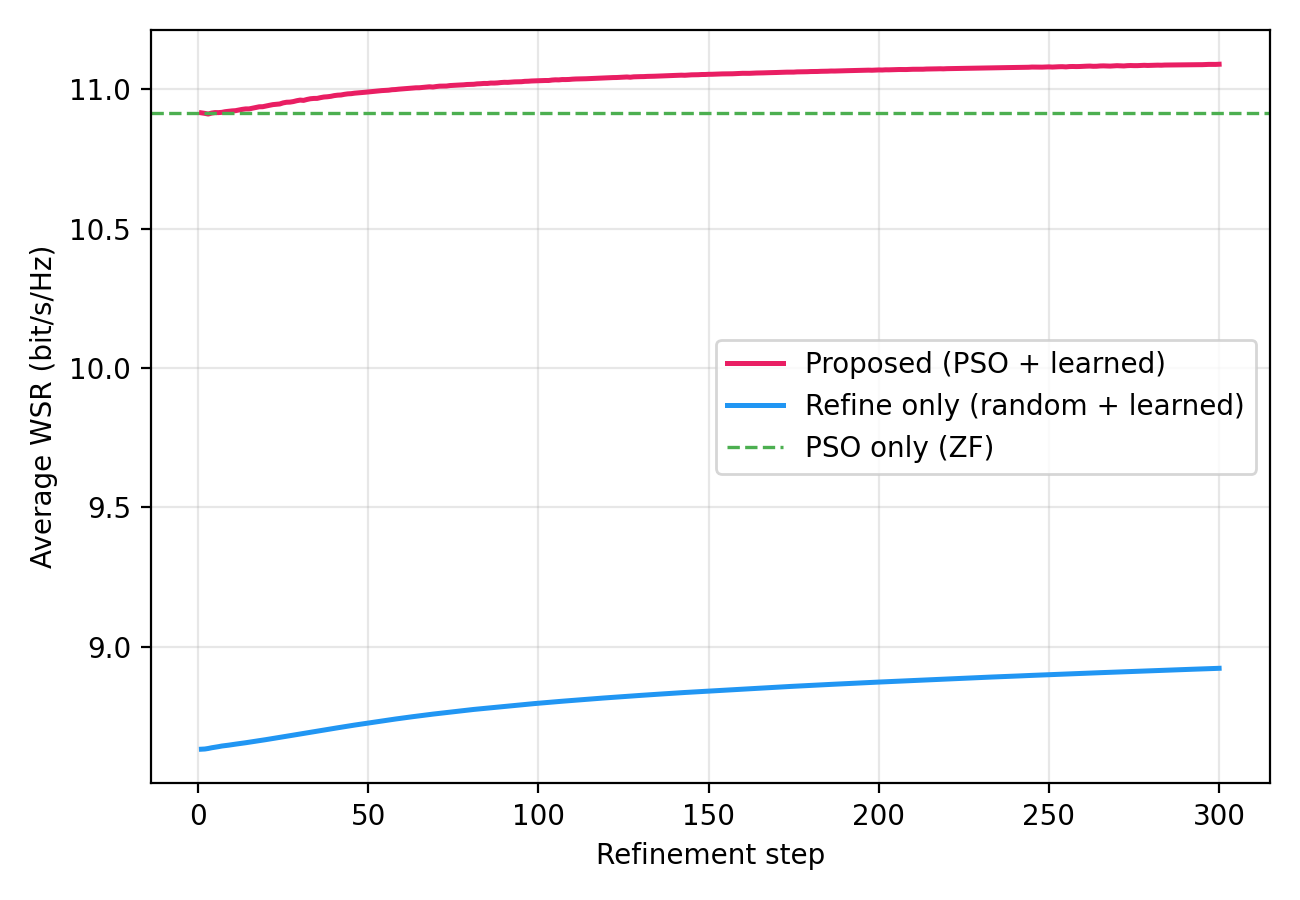}
\caption{Convergence behavior of the proposed PSA-GML algorithm and the baselines.}
\label{fig:conv}
\end{figure}

\subsection{Performance versus SNR}
Fig.~\ref{fig:snr} illustrates the WSR of all methods versus the transmit SNR. As the SNR increases, the WSR of all methods increases. In the low-SNR regime, the performance gap among the algorithms is relatively small, while in the high-SNR regime, the proposed PSA-GML algorithm achieves a more pronounced gain over the baselines. At SNR $=10$~dB, the proposed algorithm attains $11.06$~bits/s/Hz versus $9.78$~bits/s/Hz for a single AO run and $10.42$~bits/s/Hz for the AO baseline with the best of $20$ random restarts, i.e., a $13.1\%$ and a $6.2\%$ advantage, respectively, which confirms that the gain over AO is not an artifact of an unfavorable AO initialization.

\begin{figure}[t]
\centering
\includegraphics[width=0.5\textwidth]{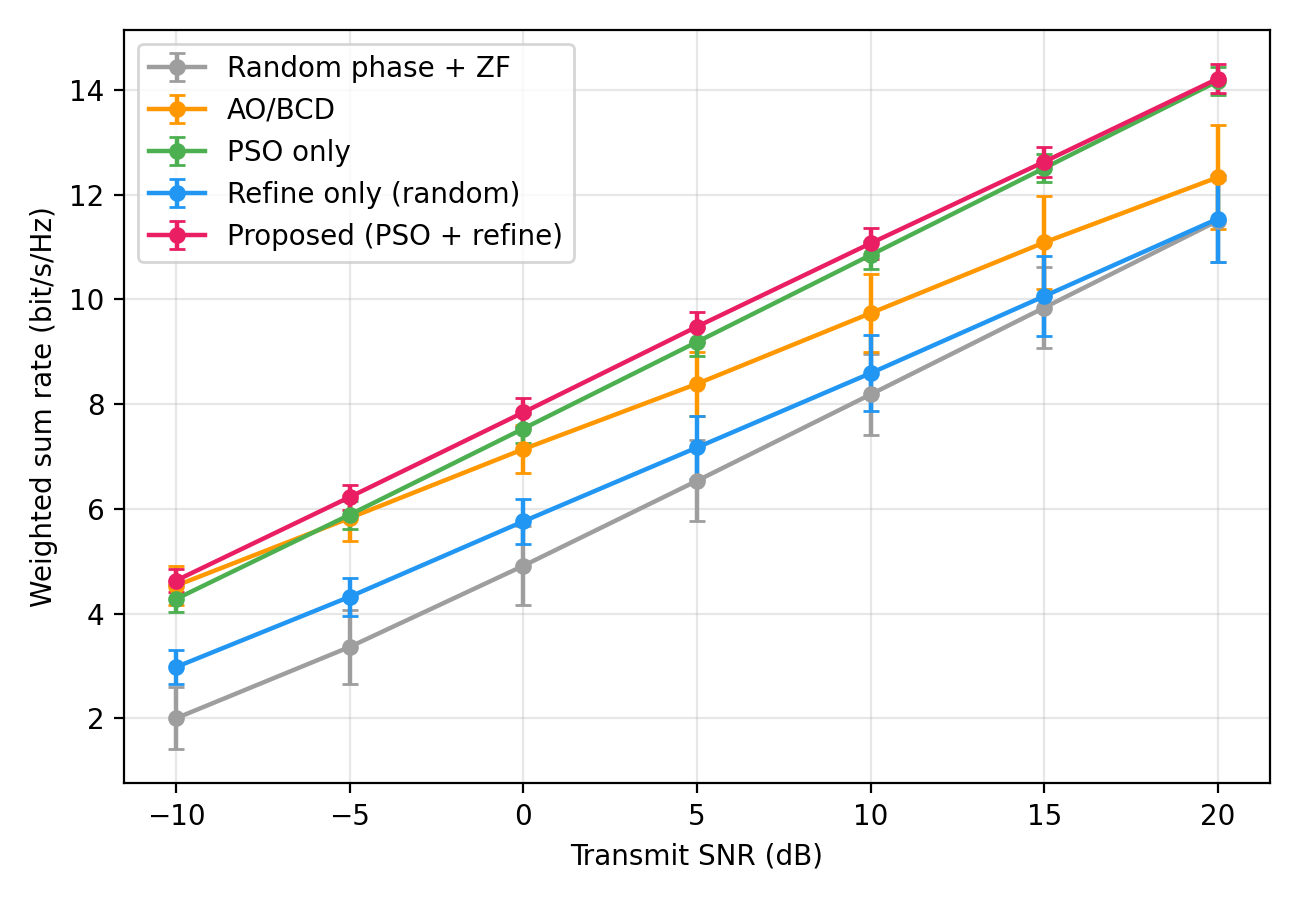}
\caption{WSR performance comparison versus the transmit SNR.}
\label{fig:snr}
\end{figure}

\subsection{Learned Optimizer in the Interference-Limited Regime}
To isolate the contribution of the meta-learning component, we evaluate Stage~2 in an interference-limited scenario with a low transmit SNR of $0$ dB and $K_t = K_r = 4$ (i.e., $K = 8$) users. The meta-optimizer is trained offline over $120$ channel realizations with a $T = 8$-step unroll and hidden size $H = 32$, and is evaluated on held-out channels with $300$ refinement steps, the same budget as the Adam reference. A short unroll ($T = 8$) is used during training to keep the computational graph shallow and the backpropagation numerically stable, whereas the deployed optimizer is unrolled for $S = 300$ steps; this is valid because the coordinate-wise LSTM learns a per-step update rule that is applied identically at every step and can therefore be safely unrolled beyond the training horizon.

Table~\ref{tab:l2o} compares (i) the PSO warm start alone, and PSO followed by (ii) a plain SGD refinement, (iii) the learned optimizer, and (iv) the hand-designed Adam refinement, all with the same $300$-step budget, and Fig.~\ref{fig:l2o} depicts the corresponding average WSR trajectories. Three observations are in order. First, under Adam the refinement stage improves the WSR by about $31\%$ relative to the PSO-only baseline. Second, the WSR of the SGD refinement ($4.32$ bits/s/Hz) is nearly identical to that of the PSO-only baseline ($4.31$ bits/s/Hz). Third, the learned optimizer, which is learned entirely from data, reaches $83.9\%$ of the Adam reference. This is notable because the hand-designed Adam here serves as a performance ceiling whose per-coordinate adaptive statistics are manually tuned, whereas the learned optimizer approaches this ceiling without any manual hyper-parameter tuning and transfers directly to unseen channels.

Table~\ref{tab:l2o} also reports the same ablation in the main regime (SNR $=10$~dB, $K=4$), and the two central observations carry over unchanged. First, plain SGD again fails to improve over the PSO warm start ($10.85$~bits/s/Hz for both), confirming that per-coordinate adaptivity---rather than the mere availability of gradient information---is the essential ingredient on the sharply peaked WSR landscape. Second, the learned optimizer attains $11.06$~bits/s/Hz, i.e., $99.4\%$ of the $11.12$~bits/s/Hz Adam reference, essentially closing the gap to the hand-designed ceiling. The refinement gain is smaller in the main regime ($+2.0\%$ over PSO) than in the interference-limited regime ($+31\%$ over PSO), which is a direct consequence of the higher quality of the PSO warm start at a moderate load and SNR: when the warm start already lies close to a locally optimal solution, the headroom left for any refinement optimizer is inherently small.

Beyond generalization to unseen channel realizations within the training distribution, the learned update rule exhibits a zero-shot transfer property across operating regimes. Applying the meta-optimizer trained exclusively in the main regime (SNR $=10$~dB, $K=4$) zero-shot to this interference-limited setting---without any fine-tuning---yields an average WSR of $4.85$~bits/s/Hz, i.e., approximately $86\%$ of the Adam reference and on par with the dedicated model in Table~\ref{tab:l2o}. This contrasts sharply with a hand-designed optimizer such as Adam, whose step size and momentum statistics must be re-tuned per scenario. The observed zero-shot transfer therefore indicates that the LSTM learns a per-coordinate update rule intrinsic to the WSR landscape rather than to any particular operating point.

\begin{table}[tp]
\centering
\caption{WSR (bits/s/Hz) of the PSO warm start followed by different refinement optimizers, in the main regime (SNR $=10$ dB, $K=4$) and the interference-limited regime (SNR $=0$ dB, $K=8$)}
\label{tab:l2o}
\begin{tabular}{|l|c|c|}
\hline
\textbf{Method} & \textbf{Main} & \textbf{Intf} \\ \hline
PSO only & 10.85 & 4.31 \\ \hline
PSO + SGD refinement & 10.85 & 4.32 \\ \hline
PSO + learned optimizer & 11.06 & 4.73 \\ \hline
PSO + Adam refinement & 11.12 & 5.64 \\ \hline
\end{tabular}
\end{table}

\begin{figure}[t]
\centering
\includegraphics[width=0.5\textwidth]{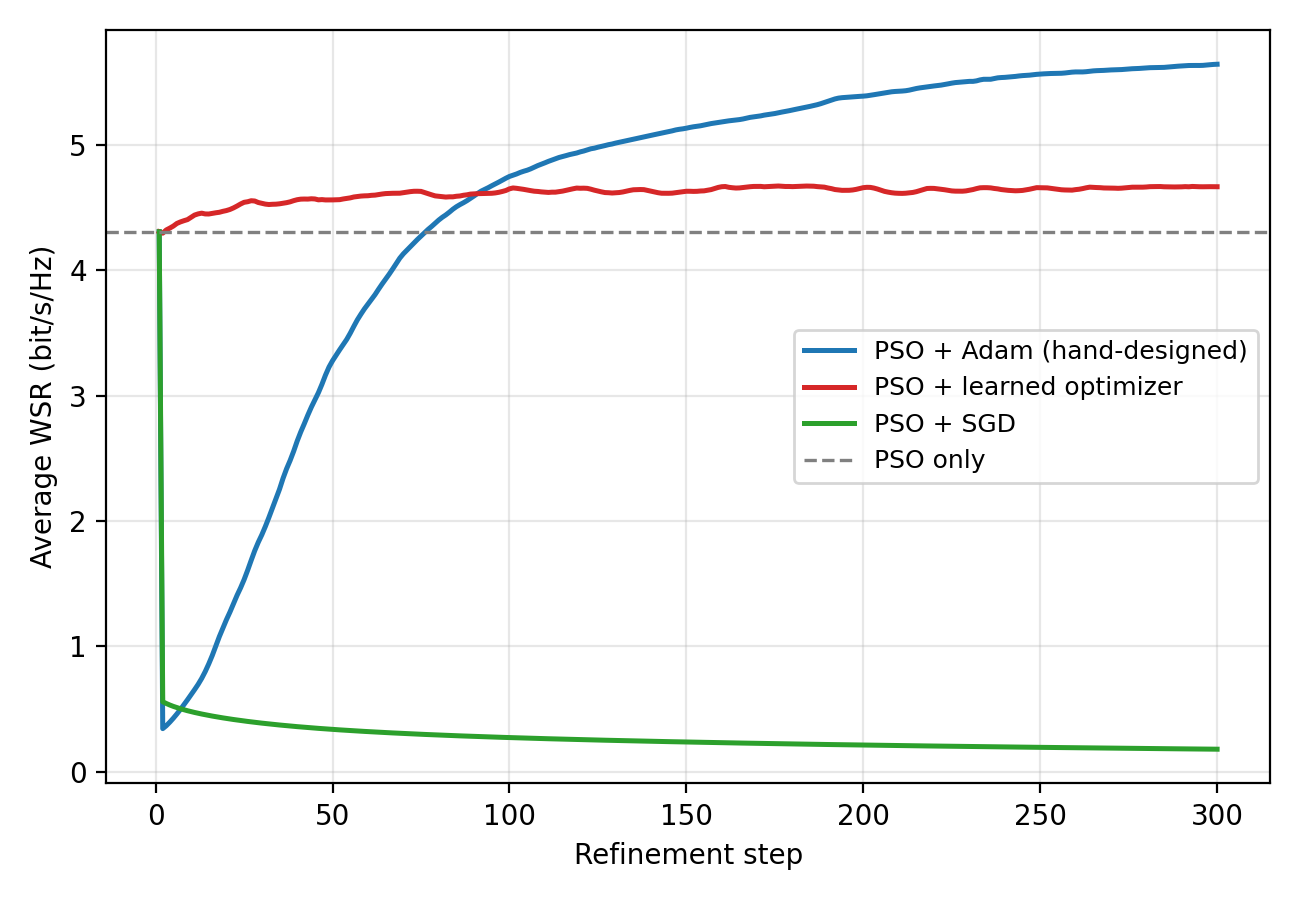}
\caption{Average WSR convergence of the PSO warm start followed by different refinement optimizers (SGD, the learned optimizer, and Adam) in the interference-limited regime.}
\label{fig:l2o}
\end{figure}

\subsection{Impact of the Number of Elements and Users}
Figs.~\ref{fig:N} and~\ref{fig:K} evaluate the WSR versus the number of STAR-RIS elements $N$ and the number of users $K$, respectively. It is observed that the WSR of all methods grows with $N$, and the proposed algorithm attains the highest WSR across the entire range of $N$, with its relative advantage over the baselines most pronounced at small-to-moderate $N$. Similarly, the proposed algorithm attains the highest sum rate for all considered values of $K$.

\begin{figure}[t]
\centering
\includegraphics[width=0.5\textwidth]{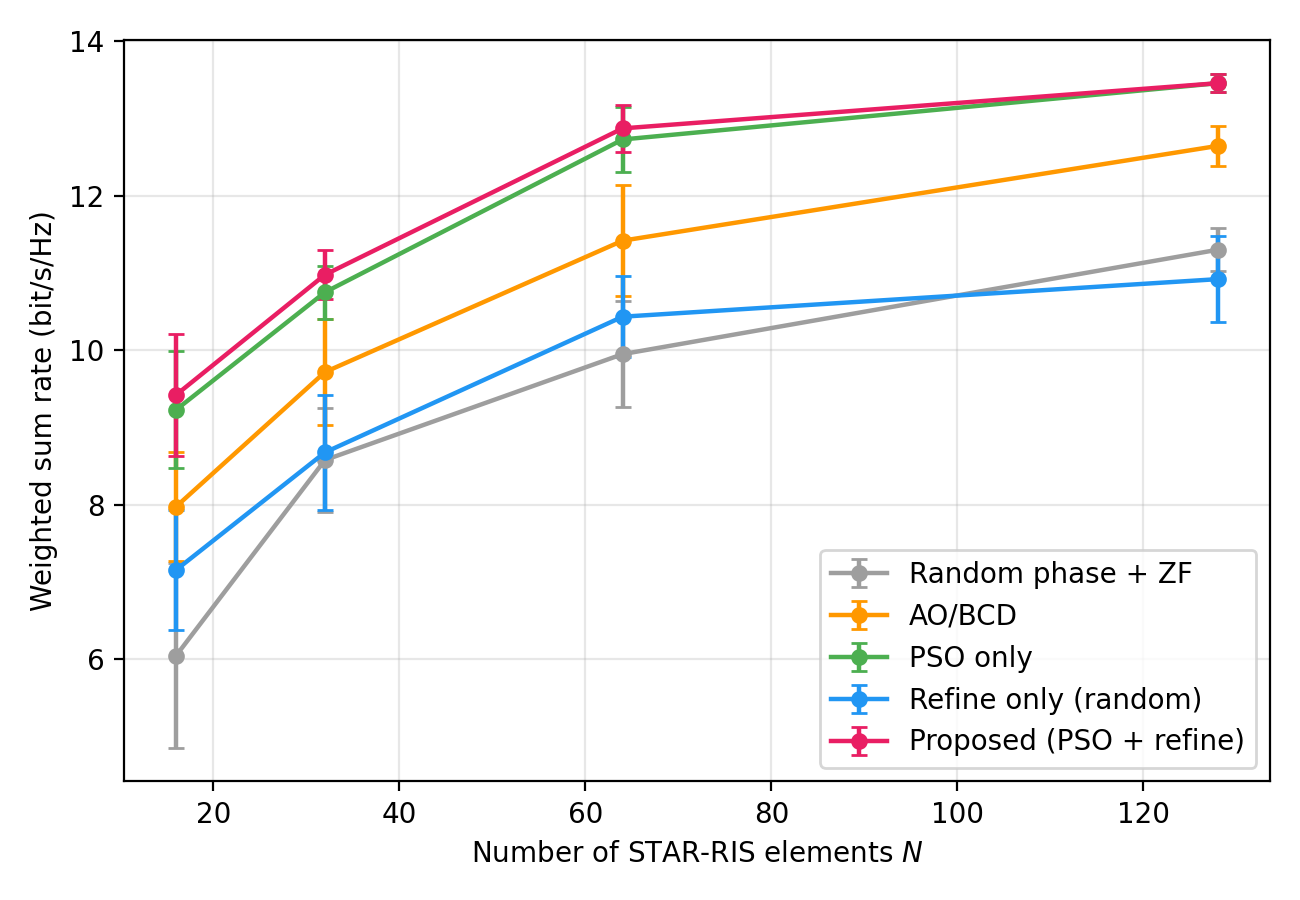}
\caption{WSR performance comparison versus the number of STAR-RIS elements $N$.}
\label{fig:N}
\end{figure}

\begin{figure}[t]
\centering
\includegraphics[width=0.5\textwidth]{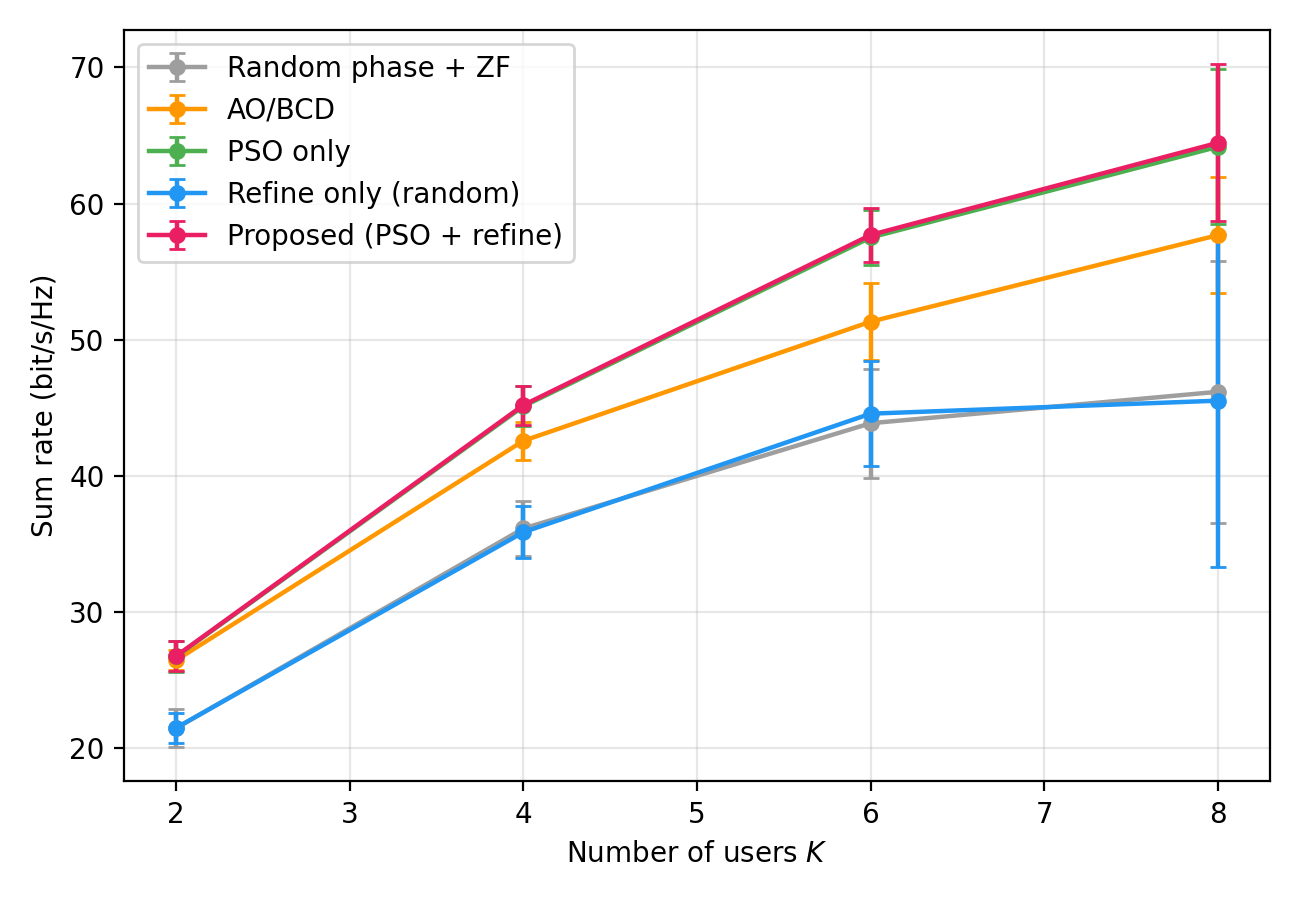}
\caption{Sum-rate performance comparison versus the number of users $K$.}
\label{fig:K}
\end{figure}

\subsection{Warm-Start Gain Analysis}
Fig.~\ref{fig:gain} investigates the gain of the PSO warm start, defined as the relative WSR improvement of the proposed algorithm over the refine-only baseline, versus $N$ under both i.i.d. Rayleigh and SV channels. Two observations are noteworthy. First, the warm-start gain is substantial under the correlated SV channel, reaching $20\%$ or more for all considered values of $N$ and peaking at $37.4\%$ at $N=16$. Second, the gain is much smaller under the i.i.d. Gaussian channel (around $10\%$ for $N\le 64$, declining to $6.9\%$ at $N=128$), and it decreases monotonically with $N$ under the correlated SV channel (from $37.4\%$ to $20.0\%$).

\begin{figure}[t]
\centering
\includegraphics[width=0.5\textwidth]{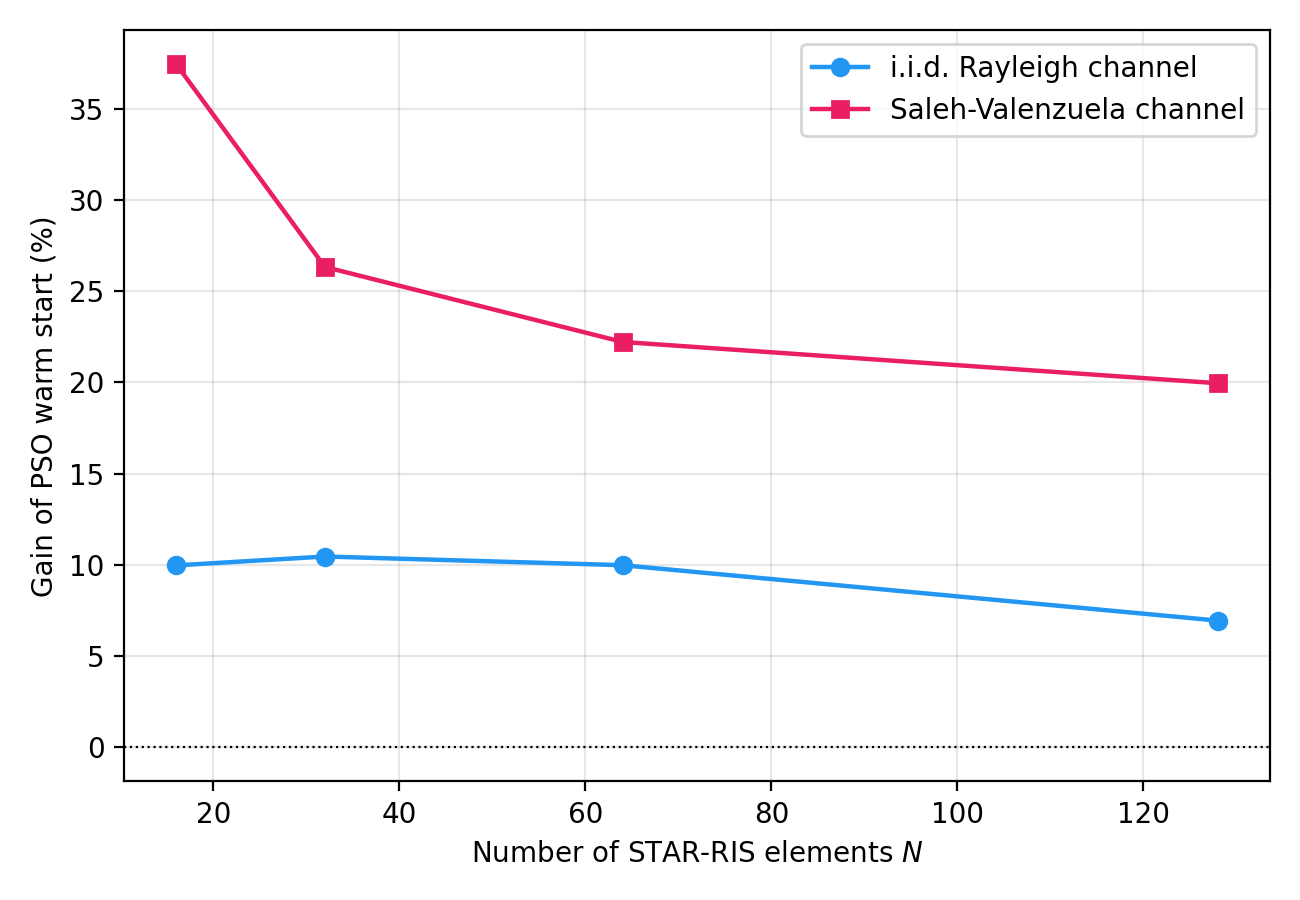}
\caption{Gain of the PSO warm start versus $N$ under i.i.d. Rayleigh and Saleh-Valenzuela channels.}
\label{fig:gain}
\end{figure}

\subsection{Runtime and Complexity}
Fig.~\ref{fig:time} compares the average runtime of the proposed algorithm and the baselines. The runtime of the proposed PSA-GML algorithm is dominated by the $S=300$ refinement steps of the learned optimizer (about $0.9$~s per realization at $N=32$, $K=4$) and exceeds that of the AO baseline (about $0.4$~s) in the current Python/Torch prototype, whose unrolled LSTM update carries a non-negligible per-step overhead. Nevertheless, as shown in Table~\ref{tab:complex}, the per-iteration complexity of the proposed algorithm and that of the AO baseline are of the same order $\mathcal{O}(NN_t K + K^3)$, and the larger observed runtime is a constant-factor effect of the unrolled LSTM in the current Python/Torch prototype: the $S$ refinement steps each carry a per-step overhead $\mathcal{O}(D H^2)$ that is not yet optimized in the prototype, and this overhead is amenable to batching and hardware acceleration. The offline training of the meta-optimizer over $N_c$ channel realizations adds a one-time cost, while the online inference only involves the PSO warm start and $S$ refinement steps.

\begin{figure}[t]
\centering
\includegraphics[width=0.5\textwidth]{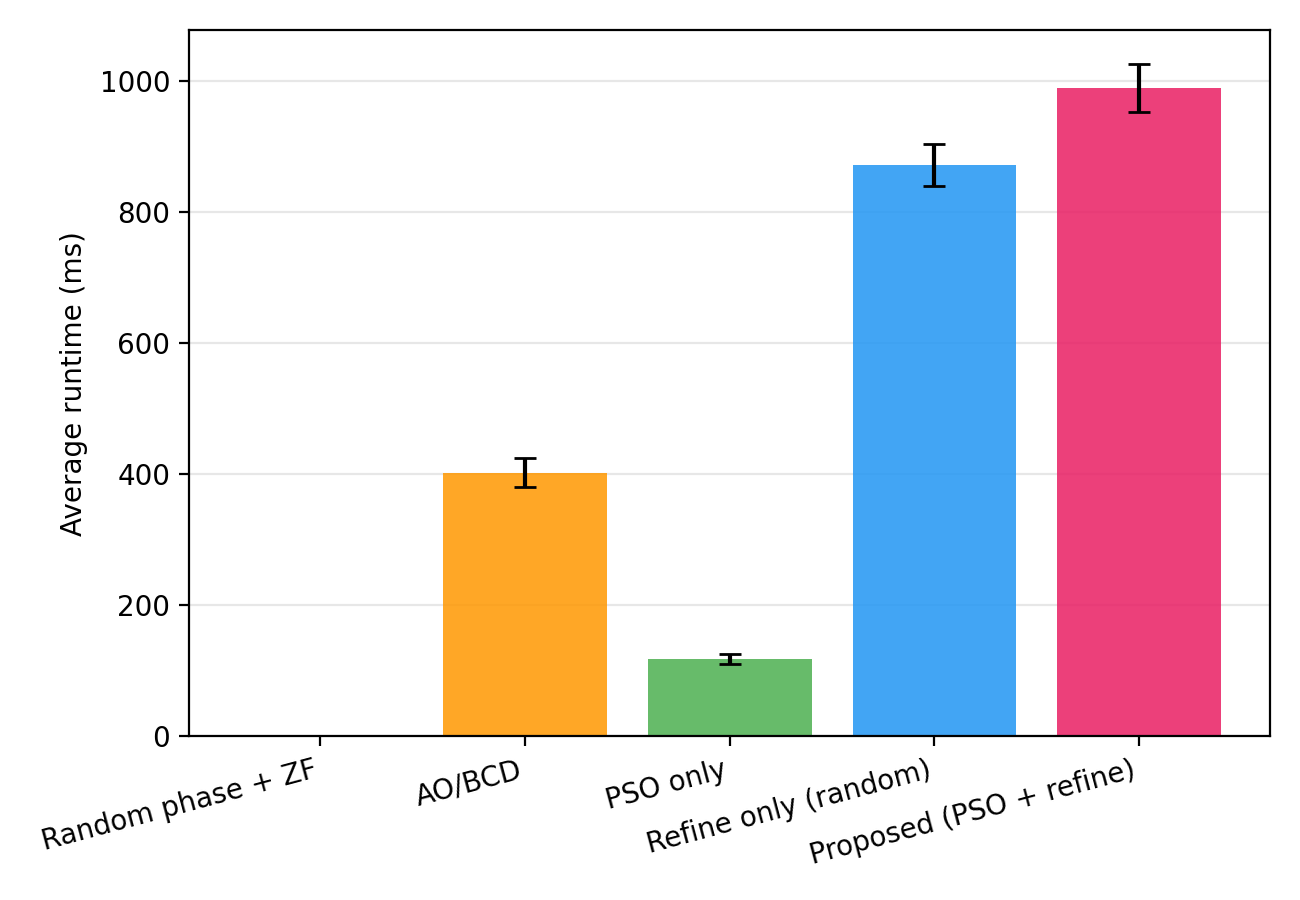}
\caption{Average runtime comparison of the proposed algorithm and the baselines.}
\label{fig:time}
\end{figure}

\begin{table*}[tp]
\centering
\caption{Computational Complexity Comparison}
\label{tab:complex}
\begin{tabular}{|l|c|c|}
\hline
 & \textbf{AO} & \textbf{Proposed PSA-GML} \\ \hline
\textbf{Per-Iteration Complexity} & $\mathcal{O}(N N_t K + K^3)$ & $\mathcal{O}(N N_t K + K^3)$ \\ \hline
\textbf{Refinement Per-Step Overhead} & -- & $\mathcal{O}(D H^2)$ \\ \hline
\textbf{Offline Training} & -- & $N_c\left[P I_{\mathrm{PSO}}\,\mathcal{O}(N N_t K + K^3) + S\left(\mathcal{O}(N N_t K + K^3) + \mathcal{O}(D H^2)\right)\right]$ \\ \hline
\textbf{Online Inference: PSO stage} & -- & $P I_{\mathrm{PSO}}\,\mathcal{O}(N N_t K + K^3)$ \\ \hline
\textbf{Online Inference: refinement / AO iterations} & $N_{\mathrm{AO}}\,\mathcal{O}(N N_t K + K^3)$ & $S\left[\mathcal{O}(N N_t K + K^3) + \mathcal{O}(D H^2)\right]$ \\ \hline
\end{tabular}
\end{table*}

\subsection{Optimized Coefficients}
Fig.~\ref{fig:coeff} visualizes the optimized amplitude-split and phase-shift coefficients obtained by the proposed algorithm for one channel realization. It is observed that the transmission and reflection powers $\beta_{t,n}^2$ and $\beta_{r,n}^2$ are adaptively balanced across elements, and the phase shifts are distributed over $[0, 2\pi)$.

\begin{figure}[t]
\centering
\includegraphics[width=0.5\textwidth]{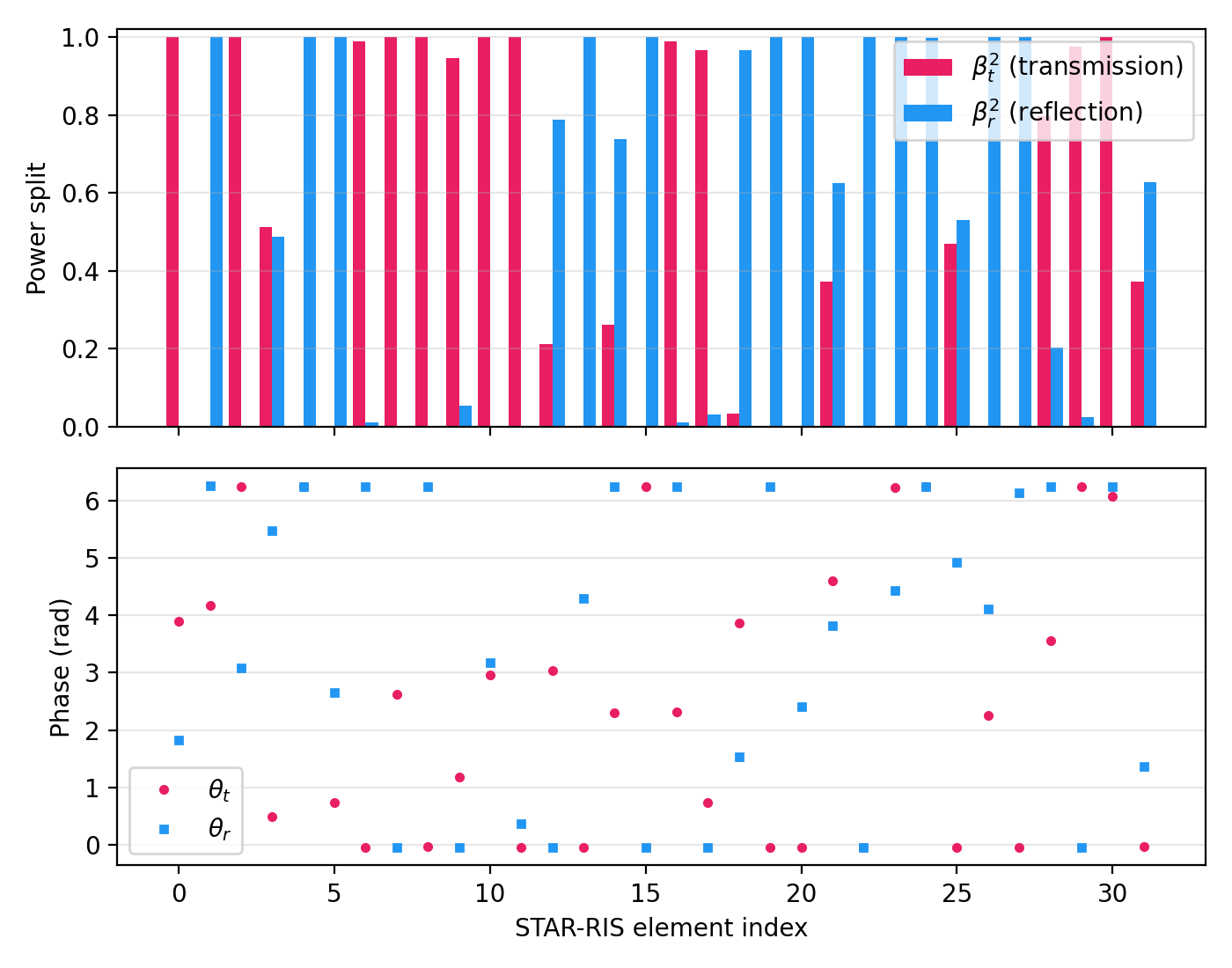}
\caption{Optimized amplitude-split and phase-shift coefficients of the STAR-RIS.}
\label{fig:coeff}
\end{figure}

\section{Conclusion}
\vspace{1.0em}

This paper addressed WSR maximization for a STAR-RIS aided multi-user downlink and introduced the particle-swarm-assisted gradient meta-learning (PSA-GML) algorithm, which combines a PSO global warm start over the STAR-RIS coefficients with a coordinate-wise LSTM meta-optimizer trained by first-order gradient meta-learning to jointly refine the coefficients and the transmit precoder. The experiments show that PSA-GML reached an 11.06 bits/s/Hz WSR, delivering a 13.1\% performance enhancement over the benchmark AO method (6.2\% over a multiple-random-restart AO) and a 35.1\% enhancement over the random-phase scheme, with robustness against initialization and a per-iteration complexity of the same order as the AO baseline. Notably, the benefit of the PSO warm start reached 20\% or more under correlated Saleh-Valenzuela channels, and the learned optimizer reached 83.9\% of the hand-designed Adam refinement in the interference-limited regime.

Future extensions include imperfect channel state information and non-line-of-sight propagation to further stress-test the algorithmic robustness, meta-learning of the PSO hyper-parameters to speed up convergence, and generalization to multi-antenna users and multiple STAR-RISs.

\end{document}